\documentclass{article}
\PassOptionsToPackage{numbers, compress}{natbib}
\usepackage{amsmath,amsfonts,bm}
\usepackage{amssymb}

\def\eqref#1{Eq.~(\ref{#1})}
\def\1{\bm{1}}

\def\vc{{\bm{c}}}

\def\vp{{\bm{p}}}
\def\vq{{\bm{q}}}
\def\vr{{\bm{r}}}

\def\vu{{\bm{u}}}
\def\vv{{\bm{v}}}

\def\vx{{\bm{x}}}
\def\vy{{\bm{y}}}
\def\vz{{\bm{z}}}

\DeclareMathAlphabet{\mathsfit}{\encodingdefault}{\sfdefault}{m}{sl}
\SetMathAlphabet{\mathsfit}{bold}{\encodingdefault}{\sfdefault}{bx}{n}

\def\gL{{\mathcal{L}}}

\def\gO{{\mathcal{O}}}

\def\sS{{\mathbb{S}}}

\newcommand{\E}{\mathbb{E}}

\DeclareMathOperator*{\var}{Var}

\usepackage{hyperref}
\usepackage{url}
\usepackage{algorithm}
\usepackage{algorithmic}
\usepackage{bbm}
\usepackage{CJKutf8}
\usepackage{graphicx}
\usepackage{subfigure}
\usepackage{booktabs}
\usepackage{multirow}
\usepackage{dsfont}
\usepackage{stackengine}
\usepackage{makecell}
\usepackage{wrapfig}
\usepackage[table,xcdraw]{xcolor}
\usepackage{enumitem}

\usepackage{titletoc}
\usepackage{lipsum}

\definecolor{mediumpersianblue}{rgb}{0.0, 0.4, 0.65}
\definecolor{citecolor}{RGB}{0, 50, 110}
\definecolor{linkcolor}{RGB}{150, 50, 50}
\hypersetup{colorlinks,linkcolor={linkcolor},citecolor={citecolor},urlcolor={citecolor}} 

\newcommand{\up}[1]{\textbf{\textcolor[rgb]{0.8 0 0}{+#1}} }

\newcommand{\revise}[1]{#1}

\newcommand{\dw}[1]{\textbf{\textcolor[rgb]{0.8 0 0}{-#1}} }

\newcommand{\method}{GH}
\newcommand{\methodspace}{GH }

\newcommand\blfootnote[1]{%
  \begingroup
  \renewcommand\thefootnote{}\footnote{#1}%
  \addtocounter{footnote}{-1}%
  \endgroup
}

\usepackage[final]{neurips_2023}

\usepackage{amsmath}
\usepackage{amssymb}
\usepackage{mathtools}
\usepackage{amsthm}

\usepackage[capitalize,noabbrev]{cleveref}

\theoremstyle{plain}
\newtheorem{theorem}{Theorem}[section]

\newtheorem{lemma}[theorem]{Lemma}

\theoremstyle{definition}
\newtheorem{definition}[theorem]{Definition}

\theoremstyle{remark}

\usepackage[textsize=tiny]{todonotes}

\title{Combating Representation Learning Disparity with Geometric Harmonization}

\author{Zhihan Zhou$^1$ \, Jiangchao Yao$^{1,2\dagger}$ \, Feng Hong$^1$ \, 
Ya Zhang$^{1,2}$ \, Bo Han$^{3}$ \, Yanfeng Wang$^{1,2\dagger}$ \\
$^{1}$Cooperative Medianet Innovation Center, Shanghai Jiao Tong University \\  $^{2}$Shanghai AI Laboratory  \quad $^{3}$Hong Kong Baptist University \\ 
\texttt{\{zhihanzhou, Sunarker, feng.hong, ya\_zhang, wangyanfeng\}@sjtu.edu.cn} \\ \quad 
\texttt{bhanml@comp.hkbu.edu.hk} \\
}

\begin{document}

\maketitle

\begin{abstract}
Self-supervised learning~(SSL) as an effective paradigm of representation learning has achieved tremendous success on various curated datasets in diverse scenarios. Nevertheless, when facing the long-tailed distribution in real-world applications, it is still hard for existing methods to capture transferable and robust representation. \revise{Conventional SSL methods, pursuing \emph{sample-level uniformity}, easily leads to representation learning disparity where head classes dominate the feature regime but tail classes passively collapse.} To address this problem, we propose a novel Geometric Harmonization~(\method) method to encourage \emph{category-level uniformity} in representation learning, which is more benign to the minority and almost does not hurt the majority under long-tailed distribution. Specially, \methodspace measures the population statistics of the embedding space on top of self-supervised learning, and then infer an fine-grained instance-wise calibration to constrain the space expansion of head classes and avoid the passive collapse of tail classes. Our proposal does not alter the setting of SSL and can be easily integrated into existing methods in a low-cost manner.
Extensive results on a range of benchmark datasets show the effectiveness of \methodspace with high tolerance to the distribution skewness. \revise{Our code is available at \href{https://github.com/MediaBrain-SJTU/Geometric-Harmonization}{https://github.com/MediaBrain-SJTU/Geometric-Harmonization}.}\blfootnote{$\dagger$ The corresponding authors are Jiangchao Yao and Yanfeng Wang.}
\end{abstract}

\section{Introduction}\label{sec:intro}

Recent years have witnessed a great success of self-supervised learning to learn generalizable representation~\citep{caron2020unsupervised,chen2020simple,doersch2015unsupervised,wang2015unsupervised}. Such rapid advances mainly benefit from the elegant training on the label-free data, which can be collected in a large volume. However, the real-world natural sources usually exhibit the long-tailed distribution~\citep{reed2001pareto}, and directly learning representation on them \revise{can} lead to the distortion issue of the embedding space, namely, the majority dominates the feature regime~\citep{zhang2021deep} and the minority collapses~\citep{mixon2022neural}. Thus, it becomes urgent to pay attention to representation learning disparity, especially as fairness of machine learning draws increasing attention~\citep{jiang2021self,liu2021self,yang2020rethinking,zhou2022contrastive}.

Different from the flourishing supervised long-tailed learning~\citep{kang2019decoupling,menon2021long, yang2020rethinking}, self-supervised learning under long-tailed distributions is still under-explored, since there is no labels available for the calibration. Existing explorations to overcome this challenge mainly resort to the possible tailed sample discovery and provide the implicit bias to representation learning. For example, BCL~\citep{zhou2022contrastive} leverages the memorization discrepancy of deep neural networks (DNNs) on unknown head classes and tail classes to drive an instance-wise augmentation. SDCLR~\citep{jiang2021self} contrasts the feature encoder and its pruned counterpart to discover hard examples that mostly covers the samples from tail classes, and efficiently enhance the learning preference towards tailed samples. DnC~\citep{tian2021divide} resorts to a divide-and-conquer methodology to mitigate the data-intrinsic heterogeneity and avoid the representation collapse of minority classes. \citet{liu2021self} adopts a data-dependent sharpness-aware minimization scheme to build support to tailed samples in the optimization. However, few works hitherto have considered the intrinsic limitation of the widely-adopted contrastive learning loss, and design the corresponding balancing mechanism to promote the representation learning parity.

\begin{figure*}[!t]
	\centering
	\includegraphics[height=0.27\textwidth, width=0.27\textwidth]{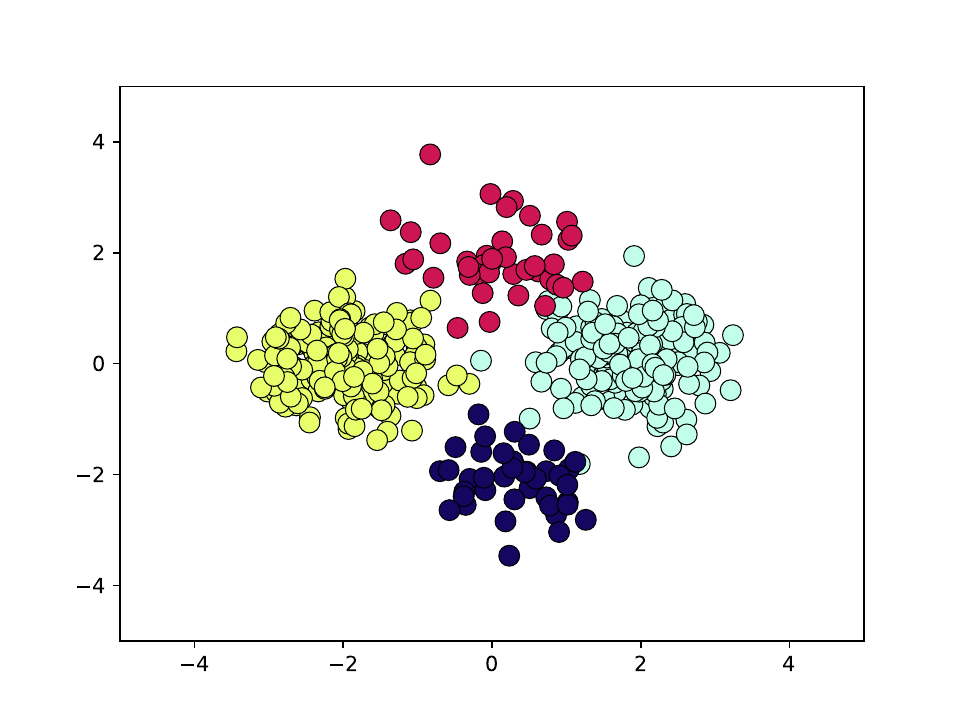}
	\hspace{7mm}
	\includegraphics[height=0.27\textwidth, width=0.27\textwidth]{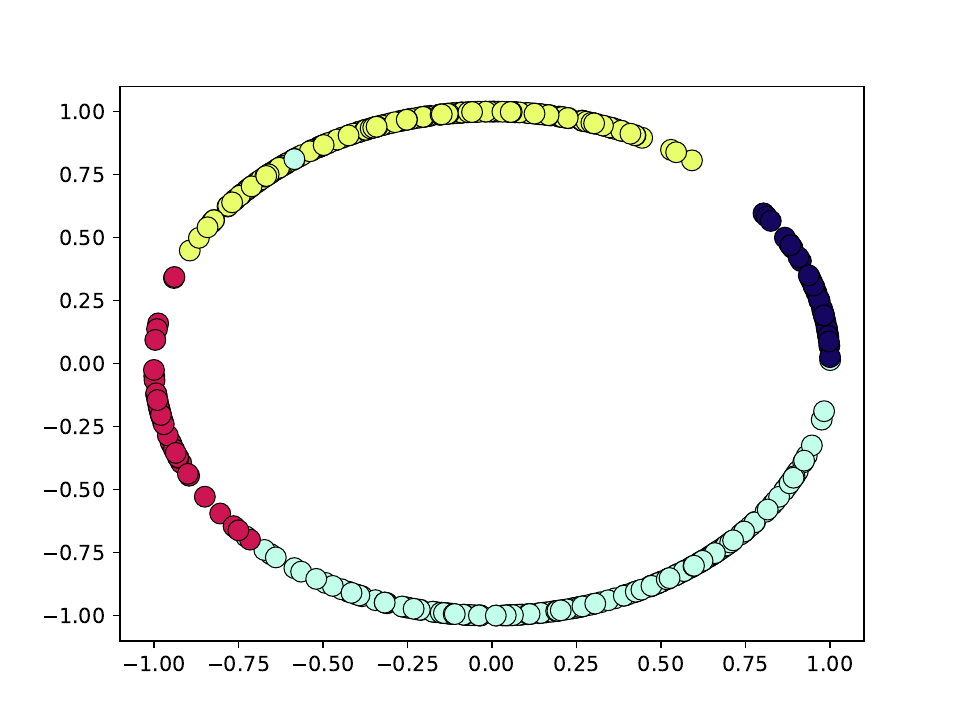}
	\hspace{7mm}
	\includegraphics[height=0.27\textwidth, width=0.27\textwidth]{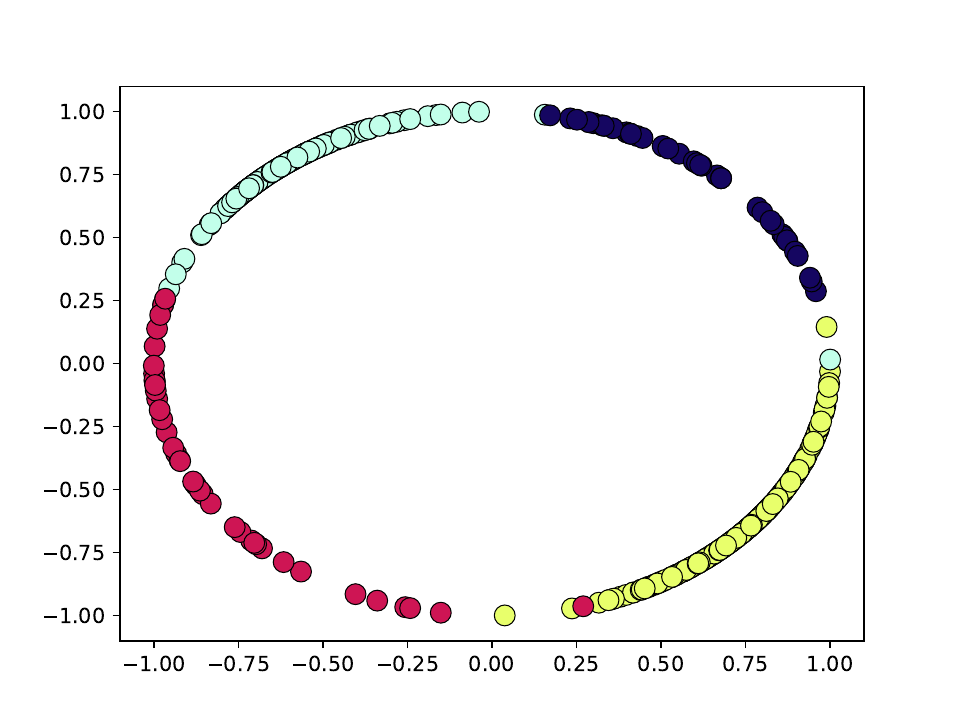} 
	\vspace{-2mm}
    \caption{Comparison of Geometric Harmonization and the plain SSL method on a 2-D imbalanced synthetic dataset. (Left) Visualization of the 2-D synthetic dataset. (Middle) The embedding distribution of each category learnt by the vanilla contrastive learning loss is approximately proportional to the \revise{number of samples}, leading to the undesired representation learning disparity.   (Right) \methodspace mitigates the adverse effect of class imbalance and approaches to the category-level uniformity.}
    \vspace{-5mm}
    \label{fig:toy}
\end{figure*}

We rethink the characteristic of the contrastive learning loss, and try to understand \emph{``Why the conventional contrastive learning underperforms in self-supervised long-tailed context?"} To answer this question, let us consider two types of representation uniformity: (1) \emph{Sample-level uniformity}. As \revise{stated} in~\citep{wang2020understanding}, the contrastive learning targets to distribute the representation of data points uniformly in the embedding space. Then, the feature span of each category is proportional to their corresponding \revise{number of samples}. (2) \emph{Category-level uniformity}. This uniformity pursues to split the region equally for different categories without considering their corresponding \revise{number of samples}~\citep{papyan2020prevalence, graf2021dissecting}. In the class-balanced scenarios, the former uniformity naturally implies the latter uniformity, resulting in the equivalent separability for classification. However, in the long-tailed distributions, they are different: sample-level uniformity leads to the feature regime that is biased towards the head classes considering their dominant sample quantity and sacrifices the tail classes due to the limited sample quantity. By contrast, category-level uniformity means the equal allocation \textit{w.r.t.} classes, which balances the space of head and tail classes, and is thus more benign to the downstream classification~\citep{fang2021exploring,graf2021dissecting,li2022targeted}. Unfortunately, there is no support for promoting category-level uniformity in contrastive learning loss, which explains the question arisen at the beginning.

In this study, we propose a novel method, termed as \emph{Geometric Harmonization} (\method) to combat representation learning disparity in SSL under long-tailed distributions. Specially, \methodspace uses a geometric uniform structure to measure the uniformity of the embedding space in the coarse granularity. Then, a surrogate label allocation is computed to provide a fine-grained instance-wise calibration, which explicitly compresses the greedy representation space expansion of head classes, and constrain the passive representation space collapse of tail classes. The alternation in the conventional loss refers to an extra efficient optimal-transport optimization that dynamically pursues the category-level uniformity. In \cref{fig:toy}, we give a toy experiment\footnote{For more details, please refer to \cref{appendix:impdetail}.}
to visualize the distribution of the embedding space without and with \method. In a nutshell, our contributions can be summarized as follows,
\begin{enumerate}
% \vspace{-5pt}
    \item  To our best knowledge, we are the first to investigate the drawback of the contrastive learning loss in self-supervised long-tailed context and point out that the resulted sample-level uniformity is an intrinsic limitation to the representation parity, motivating us to pursue category-level uniformity with more benign downstream generalization~(\cref{sec:method}). 
 
    \item We develop a novel and efficient \emph{Geometric Harmonization}~(\cref{fig:method}) to combat the representation learning disparity in SSL, which dynamically harmonizes the embedding space of SSL to approach the category-level uniformity with the theoretical guarantee.
    % \vspace{-3pt}
    \item Our method can be easily plugged into existing SSL methods for addressing the data imbalance without much extra cost. Extensive experiments on a range of benchmark datasets demonstrate the consistent improvements in learning robust representation with our \method. 
 
\end{enumerate}
\section{Related Work}
\label{sec:related-work}

\textbf{Self-Supervised Long-tailed Learning.} There are several recent explorations devoted to this direction~\citep{lin2017focal,jiang2021self,tian2021divide,liu2021self,zhou2022contrastive}. BCL~\citep{zhou2022contrastive} leverages the memorization effect of DNNs to automatically drive an instance-wise augmentation, which enhances the learning of tail samples. SDCLR~\citep{jiang2021self} constructs a self-contrast between model and its pruned counterpart to learn more balanced representation. Classic Focal loss~\citep{lin2017focal} leverages the loss statistics to putting more emphasis on the hard examples, which has been applied to self-supervised long-tailed learning~\citep{zhou2022contrastive}. DnC~\citep{tian2021divide} benefits from the parameter isolation of multi-experts during the divide step and the information aggregation during the conquer step to prevent the dominant invasion of majority.
\citet{liu2021self} proposes to penalize the sharpness surface in a reweighting manner to calibrate class-imbalance learning. \revise{Recently, TS~\citep{kukleva2023temperature} employs a dynamic strategy on the temperature factor of contrastive loss, harmonizing instance discrimination and group-wise discrimination. PMSN~\citep{assran2023hidden} proposes the power-law distribution prior, replacing the uniform prior, to enhance the quality of learned representations. }

\textbf{Hyperspherical Uniformity.} 
The distribution uniformity has been extensively explored from the physic area, \textit{e.g.}, Thomson problem~\citep{thomson1904xxiv,smale1998mathematical}, to machine learning area like some kernel-based extensions, \textit{e.g.}, Riesz s-potential~\citep{hardin2005minimal,liu2018learning} or Gaussian potential~\citep{cohn2007universally, borodachov2019discrete, wang2020understanding}. Some recent explorations regarding the features of DNNs~\citep{papyan2020prevalence,fang2021exploring,mixon2022neural} discover a terminal training stage when the embedding collapses to the geometric means of the classifier \textit{w.r.t.} each category. Specially, these optimal class means specify a perfect geometric uniform structure with clear geometric interpretations and generalization guarantees~\citep{zhu2021geometric,yang2022we,kasarla2022maximum}.
In this paper, we first extend this intuition into self-supervised learning and leverage the specific structure to combat the representation disparity in SSL.

\section{Method}

\subsection{Problem Formulation}

We denote the dataset $\mathcal{D}$, for each data point $(\vx,\vy)\in \mathcal{D}$, the input $\vx \in \mathbb{R}^m$ and the associated label  $\vy \in \{1,\dots,L\}$. Let $N$ denote the \revise{number of samples, $\mathrm{R} = {N_{max}}\slash{N_{min}}$ denote the imbalanced ratio (IR), where $N_{max},N_{min}$ is the number of samples in the largest and smallest class, respectively. Let $n_i$ denote the number of samples in class $i$.} In SSL, the ground-truth $\vy$ can not be accessed and the goal is to transform an image to an embedding via DNNs, \textit{i.e.}, $f_\theta:\mathbb{R}^m \rightarrow \mathbb{R}^d$. In the linear probing evaluation, we construct a supervised learning task with balanced datasets. A linear classifier $g(\cdot)$ is built on top of the frozen $f_\theta(\cdot)$ to produce the prediction, \textit{i.e.,} $g(f_\theta(\vx))$.

\subsection{Geometric Harmonization}\label{sec:method}
As aforementioned, most existing SSL methods in long-tailed context leverage the contrastive learning loss, which encourages the sample-level uniformity in the embedding space. 
Considering the intrinsic limitation illustrated in Figure~\ref{fig:toy}, we incorporate the geometric clues from the embedding space to calibrate the current loss, enabling our pursuit of category-level uniformity. In the following, we first introduce a specific geometric structure that is critical to Geometric Harmonization.

\begin{definition} \label{def:structure} (Geometric Uniform Structure). Given a set of vertices $\mathbf{M} \in \mathbb{R}^{d \times K}$, the geometric uniform structure satisfies the following between-class duality
\begin{equation}
     \mathbf{M}_i^\top \cdot \mathbf{M}_j =C, \ \ \forall i,j \in \{1,2,\ldots,K\},\ i\neq j,
\end{equation}
\revise{where $\left\lVert \mathbf{M}_i\right\rVert =1, \forall i \in \{1,2,\ldots,K\}$, $K$ is the number of geometric vertices and $C$ is a constant.}
\end{definition}

\begin{figure*}[!t]
	\centering

        \includegraphics[width=0.96\textwidth]{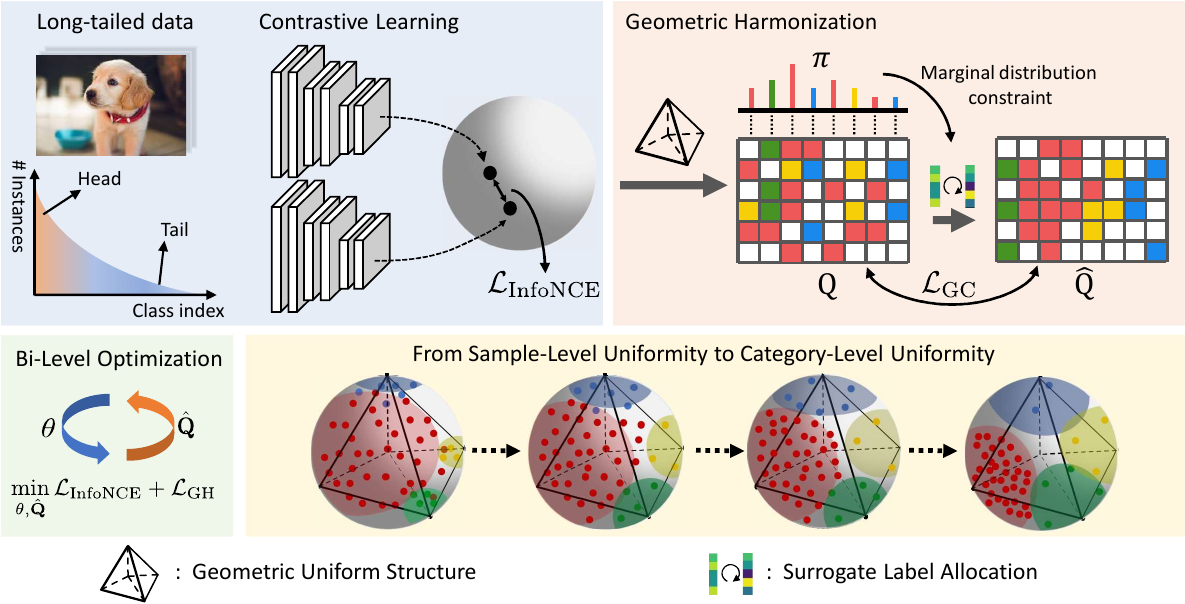}

    \caption{\textbf{Overview of Geometric Harmonization}~(\textit{w/} InfoNCE). \revise{To achieve harmonization with the category-level uniformity, we require ground-truth labels for supervision during training. However, these labels are not available under the self-supervised paradigm. Moreover, estimating surrogate labels directly from the geometric uniform structure is challenging and noisy, especially when the representation is not ideally distributed. To fullfill this gap,} we utilize the geometric uniform structure to measure the embedding space, and the captured population statistics are used for an instance-wise calibration by surrogate label allocation, which provides a supervision feedback to counteract the sample-level uniformity. Specially, our method are theoretically grounded to approach category-level uniformity at the loss minimum. The additional model parameters incurred by \methodspace are analytically and empirically demonstrated to be trained in an efficient manner.}
    \vspace{-3mm}
    \label{fig:method}
\end{figure*}

Above structure provides a characteristic that any two vectors in $\mathbf{M}$ have the same angle, namely, the unit space are equally partitioned by the vectors in $\mathbf{M}$. This fortunately follows our expectation about category-level uniformity. Specially, if we use $\mathbf{M}$ as a constant classifier to involve into training, and have the oracle label $\vy$ of the long-tailed data ($K=L$) to supervise the below prediction
\begin{equation}
    \vq = p(\vy|f_\theta(\vx),\mathbf{M}) = \exp\left(\mathbf{M}_{y}^\top \cdot f_\theta(\vx)/\gamma_{\mathrm{\method}} \right) \big/ \left(\sum_{i=1}^{K}\exp\left( \mathbf{M}_i^\top \cdot f_\theta(\vx)/\gamma_{\mathrm{\method}}\right) \right), \label{eq:gus}
\end{equation}
then according to the neural collapse theory~\citep{papyan2020prevalence}, the representation of all samples will fit the geometric uniform structure of $\mathbf{M}$ in the limit, namely, approach category-level uniformity.
However, the technical challenge {is the complete absence of annotation $\vy$ in our context, making directly constructing the objective $\min \E[\vy\log \vq]$ with Eq.~(\ref{eq:gus}) impossible to combat the representation disparity.

\label{sec:GC}
\textbf{Surrogate Label Allocation.} To address the problem of unavailable labels, we explore constructing the surrogate geometric labels $\hat{\vq}$ to supervise the training of Eq.~(\ref{eq:gus}). Concretely, we utilize the recent discriminative clustering idea~\citep{asano2020self} to acquire such geometric labels, formulated as follows 
\begin{equation} \label{eq:obj}
    \min_{\hat{\mathbf{Q}}=[\hat{\vq}_1,\dots,\hat{\vq}_N]} \mathcal{L}_{\mathrm{\method}} =  -\frac{1}{|\mathcal{D}|} \sum_{\vx_i \sim \mathcal{D}}  \hat{\vq_i} \log \vq_i,
     \ \ \ \mathrm{s.t.} \ \ \hat{\mathbf{Q}} \cdot \mathbbm{1}_N = N \cdot \boldsymbol{\pi}, \ \hat{\mathbf{Q}}^\top \cdot \mathbbm{1}_K =  \mathbbm{1}_N,
\end{equation}
where $\hat{\vq}_i$ refers to the soft assignment constrained in a $K$-dimensional probability simplex and $\boldsymbol{\pi} \in \mathbb{R}^{K}_{+}$ refers to \revise{the distribution constraint}. As we are dealing with long-tailed data, we propose to use the geometric uniform structure $\boldsymbol{\mathbf{M}}$ to automatically compute the population statistic of the embedding space as $\boldsymbol{\pi}$. Finally, \eqref{eq:obj} can be analytically solved by \textit{Sinkhorn-Knopp algorithm}~\citep{cuturi2013sinkhorn} (refer to \cref{appendix:algorithm} for \cref{algorithm:allocation}).
Note that, the above idea builds upon a conjecture: the constructed surrogate geometric labels are mutually correlated with the oracle labels so that they have the similar implication to approach category-level uniformity. We will empirically verify the rationality of such an assumption via \textit{normalized mutual information} in \cref{exp:ablation}.

\textbf{Overall Objective.} \eqref{eq:obj} can be easily integrated into previous self-supervised long-tailed learning methods for geometric harmonization, \textit{e.g.,} SDCLR~\citep{jiang2021self} and BCL~\citep{zhou2022contrastive}. For simplicity, we take their conventional InfoNCE loss~\citep{oord2018representation} as an example and write the overall objective as follows
\begin{equation}\label{eq:overall}
    \min_{\theta,\hat{\mathbf{Q}}} \mathcal{L} = \mathcal{L}_{\mathrm{InfoNCE}} + w_{\mathrm{\method}} \mathcal{L}_{\mathrm{\method}},
\end{equation}
\revise{where $w_{\mathrm{\method}}$ represents the weight of the geometric harmonization loss.} Optimizing \eqref{eq:overall} refers to a bi-level optimization style: in the outer-loop, optimize $\min_{\hat{\mathbf{Q}}} \mathcal{L}_{\mathrm{\method}}$ with fixing $\theta$ to compute the surrogate geometric labels; in the inner-loop, optimize $\min_{\theta} \mathcal{L}_{\mathrm{InfoNCE}} + \mathcal{L}_{\mathrm{\method}}$ with fixing $\hat{\mathbf{Q}}$ to learn the representation model. \revise{The additional cost compared to} the vanilla method \revise{primarily arises from} from the outer-loop, which will be discussed in \cref{sec:method-analysis} and verified in \cref{exp:ablation}.

\begin{wraptable}{r}{0.58\linewidth}
% \vspace{-14pt}
\vspace{-3pt}
\centering
\caption{Linear probing of vanilla discriminative clustering methods and variants of \methodspace on CIFAR-100-LT.}\label{tab:clustering}
\resizebox{0.58\textwidth}{!}{
\begin{tabular}{c|c|cc|ccc}
\toprule[1.5pt]
IR       & SimCLR & +SeLA & +SwAV & +\method & \textit{w/o} GUS & \textit{w/o} $\boldsymbol{\pi}$     \\ \midrule[0.6pt]\midrule[0.6pt]
100 & 50.7 & 50.5 & 52.2 & 54.0 &53.3 &53.1  \\
50  & 52.2 & 52.0 & 53.0 & 55.4 &54.6 &54.4  \\
10  & 55.7  & 56.0 & 56.1 & 57.4 &56.7 &56.4  \\ \bottomrule[1.5pt]   
\end{tabular}}
\end{wraptable}
Compared with previous explorations~\citep{asano2020self,caron2020unsupervised}, the uniqueness of \methodspace lies in the following three aspects: (1) \emph{Geometric Uniform Structure}. The pioneering works mainly resort to a learnable classifier to perform clustering, which can easily be distorted in the long-tailed scenarios~\citep{fang2021exploring}. Built on the geometric uniform classifier, our method is capable to provide high-quality clustering results with clear geometric interpretations. (2) \emph{Flexible Class Prior}. The class prior $\boldsymbol{\pi}$ is assumed to be uniform among the previous attempts. When moving to the long-tailed case, this assumption will strengthen the undesired sample-level uniformity. In contrast, our methods can potentially cope with any distribution with the automatic surrogate label allocation. (3) \emph{Theoretical Guarantee.} \methodspace is theoretically grounded to achieve the category-level uniformity in the long-tailed scenarios (refer to \cref{sec:theory}), which has never been studied in previous methods. To gain more insights into our method, we further compare GH with discriminative clustering methods~(SeLA~\citep{asano2020self}, SwAV~\citep{caron2020unsupervised}) and investigate the impact of various components in \method. From the results in \cref{tab:clustering}, we can see that \methodspace consistently outperforms the vanilla discriminative clustering baselines in the linear probing evaluation. Notably, we observe that both GUS and the class prior $\boldsymbol{\pi}$ play a critical role in our \method.

\subsection{Theoretical Understanding}\label{sec:theory}

Here, we reveal the theoretical analysis of \methodspace on promoting the representation learning to achieve category-level uniformity instead of sample-level uniformity. Let us begin with a deteriorated case of sample-level uniformity under the extreme imbalance, \textit{i.e.}, minority collapse~\citep{fang2021exploring}.

\begin{lemma}\label{lemma:mc}
    (Minority collapse) Assume the samples follow the uniform distribution \revise{$n_1=n_2=\dots=n_{L_H}=n_H$, $n_{L_H+1}=n_{L_H+2}=\dots=n_{L}=n_T$} in head and tail classes respectively. Assume $d \geq L$ and $n_H/n_T \rightarrow +\infty$, the lower bound~(\cref{lemma:lower}) of $\mathcal{L}_{\mathrm{InfoNCE}}$ achieves the minimum when the class means of the tail classes collapse to an identical vector:
    \begin{equation}
        \lim \boldsymbol{\mu}_i - \boldsymbol{\mu}_{j} = \mathbf{0}_L, \ \forall L_H \leq i \leq j \leq L, \nonumber
    \end{equation}
    where $\boldsymbol{\mu}_l = \frac{1}{n_l}\sum_{i=1}^{n_l} f_\theta(\vx_{l,i})$ denotes the class means and $\vx_{l,i}$ is the $i$-th data point with label $l$.
\end{lemma}

This phenomenon indicates all representations of the minority will collapse completely to one point without considering the category discrepancy, which \revise{aligns with} our observation regarding the passive collapse of tailed samples in Figure~\ref{fig:toy}. 
To further theoretically analyze \method, we first quantitatively define category-level uniformity in the following, and then theoretically claim that with the geometric uniform structure~(Definition.~\ref{def:structure}) and the perfect aligned allocation~(\eqref{eq:obj}), we can achieve the loss minimum at the stage of realizing category-level uniformity.

\begin{definition}\label{def:cu}
    (Categorical-level Uniformity)  We define categorical-level uniformity on the embedding space \textit{w.r.t} the geometric uniform structure $\mathbf{M}$ when it satisfies
    \begin{equation}
        \boldsymbol{\mu}_k^*=\mathbf{M}_k, \ \forall k=1,2,\dots,K, \nonumber
    \end{equation}
    where $\boldsymbol{\mu}_k^* = \frac{1}{n_k}\sum_{i=1}^{n_k} f_\theta^*(\vx_{k,i})$ represents the class mean for samples assigned with the surrogate geometric label $k$ in the embedding space.
\end{definition}

\begin{theorem}\label{theorem:opt}
(Optimal state for $\mathcal{L}$) Given Eq.~(\ref{eq:overall}) under the proper optimization strategy, when it arrives at the category-level uniformity~(\cref{def:cu}) defined on the geometric uniform structure $\mathbf{M}$~(\cref{def:structure}), we will achieve the minimum of the overall loss $\mathcal{L}^*$ as
\begin{equation}
    \mathcal{L}^* = -2\sum_{l=1}^{L}\boldsymbol{\pi}_{l}^{\vy}\log\left(1/\left(1+(K-1)\exp(C-1)\right)\right) + \log\left(J/L\right),
\end{equation}
where $J$ denotes the size of the collection of the negative samples and $\boldsymbol{\pi}^{\vy}$ refers to the marginal distribution of the latent ground-truth labels $\vy$.
\end{theorem}

This guarantees the desirable solution with the minimal intra-class covariance and the maximal inter-class covariance under the geometric uniform structure~\cite{papyan2020prevalence,fang2021exploring}, which benefits the downstream generalization. Notably, no matter the data distribution is balanced or not, our method can persistently maintain the theoretical merits on calibrating the class means to achieve category-level uniformity. We also empirically demonstrate the comparable performance with \methodspace on the balanced datasets in Section~\ref{exp:ablation}, as in this case category-level uniformity is equivalent to sample-level uniformity.

\subsection{Implementation and Complexity analysis}\label{sec:method-analysis}
In \cref{algorithm:method} of \cref{appendix:algorithm}, we give the complete implementation of our method. One point that needs to be clarified is that we learn the label allocation $\hat{\vq}$ in the mini-batch manner. 
In addition, the geometric prediction $\vq$ and the adjusted $\hat{\vq}$ are computed at the beginning of every epoch as the population-level statistic will not change much in a few mini-batches. Besides, we maintain a momentum update mechanism to track the prediction of each sample to stabilize the training, \textit{i.e.}, $\vq^{m} \leftarrow \beta \vq^{m} + (1-\beta)\vq$. When combined with the joint-embedding loss, we naturally adopt a cross-supervision mechanism $\min \E[\hat{\vq}^{+}\log \vq]$ for the reconciliation with contrastive baselines. The proposed method is illustrated in Figure~\ref{fig:method} for visualization.

For complexity, 
assume that the standard optimization of deep neural networks requires forward and backward step in each mini-batch update with the time complexity as $\mathcal{O}(B\Lambda)$, where $B$ is the mini-batch size and $\Lambda$ is the parameter size. At the parameter level, we add an geometric uniform structure with the complexity as $\mathcal{O}(BKd)$, where $K$ is the number of geometric labels and $d$ is the embedding dimension. For Sinkhorn-Knopp algorithm, it only refers to a simple matrix-vector multiplication as shown in \cref{algorithm:allocation}, whose complexity is $\mathcal{O}(E_s(B+K+BK))$ with the iteration step $E_s$. The complexity incurred in the momentum update is $\mathcal{O}(BK)$. Since $K,d$ and $E_s$ are significantly smaller than the model parameter $\Lambda$ of a million scale, the computational overhead involved in \methodspace is negligible compared to $\mathcal{O}(B\Lambda)$. The additional storage for a mini-batch of samples is the matrix $\mathbf{Q}^{m} \in \mathbb{R}^{K \times B}$, which is also negligible to the total memory usage. To the end, \methodspace incurs only a small computation or memory cost and thus can be plugged to previous methods in a low-cost manner. The empirical comparison about the computational cost is summarized in Table~\ref{tab:ccost}.

\section{Experiments}
\label{sec:Experiments}
\subsection{Experimental Setup} \label{sec:exp}

\textbf{Baselines.} \revise{We mainly choose five baseline methods, including (1) \textit{plain contrastive learning}: SimCLR~\citep{chen2020simple}, (2) \textit{hard example mining}: Focal~\citep{lin2017focal}, (3) \textit{asymmetric network pruning}: SDCLR~\citep{jiang2021self}, (4) \textit{multi-expert ensembling}: DnC~\citep{tian2021divide}, (5) \textit{memorization-guided augmentation}: BCL~\citep{zhou2022contrastive}. Empirical comparisons with more baseline methods can be referred to \cref{appendix:moresslbaseline}}.

\textbf{Implementation Details.} 
% We implement all experiments in PyTorch~\citep{paszke2019pytorch}. 
Following previous works~\citep{jiang2021self, zhou2022contrastive}, we use ResNet-18~\citep{he2016deep} as the backbone for small-scale dataset~(CIFAR-100-LT~\citep{cao2019learning}) and ResNet-50~\citep{he2016deep} for large-scale datasets~(ImageNet-LT~\citep{liu2019large}, Places-LT~\citep{liu2019large}). For experiments on CIFAR-100-LT, we train model with the SGD optimizer, batch size 512, momentum 0.9 and weight decay factor $5 \times 10^{-4}$ for 1000 epochs. For experiments on ImageNet-LT and Places-LT, we only train for 500 epochs with the batch size 256 and weight decay factor $1 \times 10^{-4}$. For learning rate schedule, we use the cosine annealing decay with the learning rate $0.5 \rightarrow 1e^{-6}$ for all the baseline methods. As \methodspace is combined with baselines, a proper warming-up of 500 epochs on CIFAR-100-LT and 400 epochs on ImageNet-LT and Places-LT are applied. The cosine decay is set as $0.5 \rightarrow 0.3$, $0.3 \rightarrow 1e^{-6}$ respectively. For hyper-parameters of \method, we provide a default setup across all the experiments: set the geometric dimension $K$ as 100, \revise{$w_{\mathrm{\method}}$ as 1} and the temperature $\gamma_{\mathrm{\method}}$ as 0.1. In the surrogate label allocation, we set the regularization coefficient $\lambda$ as 20 and Sinkhorn iterations $E_s$ as 300. Please refer to \cref{appendix:impdetail} for more experimental details.

\textbf{Evaluation Metrics.} Following~\citep{jiang2021self,zhou2022contrastive}, \emph{linear probing} on a balanced dataset is used for evaluation. We conduct full-shot evaluation on CIFAR-100-LT and few-shot evaluation on ImageNet-LT and Places-LT. For comprehensive performance comparison, we present the linear probing performance and the standard deviation among three disjoint groups, \textit{i.e.}, [many, medium, few] partitions~\citep{liu2019large}.

\subsection{Linear Probing Evaluation}

\definecolor{greyL}{RGB}{235,235,235}
\begin{table*}[!t]
\centering
\vspace{-3pt}
\caption{Linear probing results on CIFAR-100-LT with different imbalanced ratios~(100, 50, 10), ImageNet-LT and Places-LT. Many/Med/Few~($\uparrow$) indicate the average accuracy~($\%$) \textit{w.r.t} fine-grained partitions according to the class cardinality. Std~($\downarrow$) means the standard deviation of the group-level performance and Avg~($\uparrow$) is the average accuracy~($\%$) of the full test set. Following the previous work~\cite{jiang2021self,zhou2022contrastive}, Std represents a balancedness measure to quatify the variance among three specified groups. Improv.~($\uparrow$) represents the averaging performance improvements \textit{w.r.t.} different baseline methods. \textbf{We report the error bars with the multi-run experiments in \cref{tab:errorbar} in \cref{appendix:results}}.}\label{tab:main}

\resizebox{\textwidth}{!}{%
\begin{tabular}{c|c|cc|cc|cc|cc|cc|c}
\toprule[1.5pt]
\multicolumn{2}{c|}{Dataset}                & SimCLR & +\method  & Focal & +\method  & SDCLR & +\method  & DnC   & +\method & BCL   & +\method & \textbf{Improv.} \\\midrule[0.6pt]\midrule[0.6pt]
\multirow{5}{*}{\rotatebox[origin=c]{90}{CIFAR-R100}}          & Many & 54.97 & 57.38                          & 54.24 & 57.01                          & 57.32 & 57.44           & 55.41   & 57.56            & 59.15 & 59.50      & \up{1.56}               \\ 
                                    & Med  & 49.39  & 52.27                           & 49.58 & 52.93                          & 50.70 & 52.85        & 51.30    & 53.74                  & 54.82 & 55.73    & \up{2.35}                       \\ 
                                    & Few  & 47.67  & 52.12                           & 49.21 & 51.74                          & 50.45 & 54.06        & 50.76    & 53.26                  & 55.30 & 57.67    & \up{3.09}                       \\ 
                                    & Std  & 3.82   & 2.99                            & 2.80  & 2.76                            & 3.90  & 2.38        & 2.54    &  2.36                  & 2.37  & 1.89      & \dw{0.61}                      \\ 
                                    & \cellcolor{greyL}Avg  & \cellcolor{greyL}50.72  & \cellcolor{greyL}53.96                           & \cellcolor{greyL}51.04 & \cellcolor{greyL}53.92                           & \cellcolor{greyL}52.87 & \cellcolor{greyL}54.81          & \cellcolor{greyL}52.52 & \cellcolor{greyL}54.88                 & \cellcolor{greyL}56.45 & \cellcolor{greyL}57.65    & \cellcolor{greyL}\up{2.32}                      \\\midrule[0.6pt]
\multirow{5}{*}{\rotatebox[origin=c]{90}{CIFAR-R50}}           & Many & 56.00  & 58.88                           & 55.40 & 57.97                           & 57.50 & 58.47            & 56.03    &  59.04             & 59.44 & 60.82    & \up{2.16}                       \\
                                    & Med  & 50.48  & 53.00                           & 51.14 & 53.55                           & 51.85 & 53.88       & 52.68    &  55.05                  & 54.73 & 57.58        & \up{2.44}                  \\
                                    & Few  & 50.12  & 54.27                           & 50.02 & 53.58                           & 52.15 & 53.58       & 50.83    &  54.81                  & 57.30 & 58.55        & \up{2.87}                  \\
                                    & Std  & 3.30   & 3.09                            & 2.84  & 2.54                            & 3.18  & 2.74        & 2.64    &   2.38                 & 2.36  & 1.66           & \dw{0.38}                 \\
                                    & \cellcolor{greyL}Avg  & \cellcolor{greyL}52.24  & \cellcolor{greyL}55.42                           & \cellcolor{greyL}52.22 & \cellcolor{greyL}55.06                         & \cellcolor{greyL}53.87 & \cellcolor{greyL}55.34            & \cellcolor{greyL}53.21 & \cellcolor{greyL}56.33              & \cellcolor{greyL}57.18 & \cellcolor{greyL}59.00     & \cellcolor{greyL}\up{2.49}                      \\\midrule[0.6pt]
\multirow{5}{*}{\rotatebox[origin=c]{90}{CIFAR-R10}}           & Many & 57.85  & 59.26                           & 58.18 & 60.06                          & 58.47 & 59.21             & 59.82    &  61.09            & 60.41 & 61.41     & \up{1.26}                     \\
                                    & Med  & 55.06  & 56.91                           & 55.82 & 56.79                           & 54.79 & 56.06       & 56.67    &  58.33                  & 57.15 & 59.27        & \up{1.57}                   \\
                                    & Few  & 54.03  & 55.85                           & 54.64 & 57.24                           & 52.97 & 55.58       & 56.21    &  57.33                  & 59.76 & 60.30        & \up{1.74}               \\
                                    & Std  & 1.98   & 1.75                            & 1.80  & 1.77                            & 2.80  & 1.97        & 1.96    &   1.95                  & 1.73  & 1.07          & \dw{0.35}                \\
                                    & \cellcolor{greyL}Avg  & \cellcolor{greyL}55.67  & \cellcolor{greyL}57.36                          & \cellcolor{greyL}56.23 & \cellcolor{greyL}58.05                         & \cellcolor{greyL}55.44 & \cellcolor{greyL}56.97             & \cellcolor{greyL}57.59 & \cellcolor{greyL}58.94              & \cellcolor{greyL}59.12 & \cellcolor{greyL}60.34     & \cellcolor{greyL}\up{1.52}                    \\\midrule[0.6pt]\midrule[0.6pt]
\multirow{5}{*}{\rotatebox[origin=c]{90}{ImageNet-LT}}        & Many & 41.69  & 41.53                          & 42.04 & 42.55                          & 40.87 & 41.92                & 41.70    &  42.19         & 42.92 & 43.22         & \up{0.44}                  \\
                            & Med  & 33.96  & 36.35                           & 35.02 & 36.75                           & 33.71 & 36.53               & 34.68    &  36.63          & 35.89 & 38.16               & \up{2.23}           \\
                            & Few  & 31.82  & 35.84                           & 33.32 & 36.28                           & 32.07 & 36.04               & 33.58    &  35.86          & 33.93 & 36.96               & \up{3.25}           \\
                            & Std  & 5.19   & 3.15                            & 4.62  & 3.49                            & 4.68  & 3.26                & 4.41    &   3.45       & 4.73  & 3.32                    & \dw{1.39}      \\
                            &\cellcolor{greyL}Avg  &\cellcolor{greyL}36.65  &\cellcolor{greyL}38.28                           &\cellcolor{greyL}37.49 & \cellcolor{greyL}38.92                          &\cellcolor{greyL}36.25 & \cellcolor{greyL}38.53                        & \cellcolor{greyL}37.23 & \cellcolor{greyL}38.67              &\cellcolor{greyL}38.33 &\cellcolor{greyL}39.95       & \cellcolor{greyL}\up{1.68}                    \\\midrule[0.6pt]\midrule[0.6pt]
\multirow{5}{*}{\rotatebox[origin=c]{90}{Places-LT}}   & Many & 31.98  & 32.46                          & 31.69 & 32.40                           & 32.17 & 32.78               & 32.07    &  32.51         & 32.69 & 33.22                & \up{0.55}           \\
                            & Med  & 34.05  & 35.03                           & 34.33 & 35.14                           & 34.71 & 35.60               & 34.51    &  35.55          & 35.37 & 36.00               & \up{0.87}            \\
                            & Few  & 35.63  & 36.14                           & 35.73 & 36.49                           & 35.69 & 36.18               & 35.84   &  35.91          & 37.18 & 37.62                & \up{0.45}           \\
                            & Std  & 1.83   & 1.89                           & 2.05  & 2.08                           & 1.82  & 1.82                  & 1.91    &  1.87           & 2.26  & 2.23                 & 0.00           \\
                            & \cellcolor{greyL}Avg  & \cellcolor{greyL}33.61  & \cellcolor{greyL}34.33                           & \cellcolor{greyL}33.65 & \cellcolor{greyL}34.42                           & \cellcolor{greyL}33.99 & \cellcolor{greyL}34.70                  & \cellcolor{greyL}33.90 & \cellcolor{greyL}34.52              & \cellcolor{greyL}34.76 & \cellcolor{greyL}35.32     & \cellcolor{greyL}\up{0.68}                 \\\bottomrule[1.5pt]    

\end{tabular}
}\vspace*{-10pt}
\end{table*}

\textbf{CIFAR-100-LT.} In \cref{tab:main}, we summarize the linear probing performance of baseline methods \textit{w/} and \textit{w/o} \methodspace on a range of benchmark datasets, and provide the analysis as follows.

(1) \emph{Overall Performance.} \methodspace achieves the competitive results \textit{w.r.t} the [many, medium, few] groups, yielding a overall performance improvements averaging as 2.32$\%$, 2.49$\%$ and 1.52$\%$ on CIFAR-100-LT with different imbalanced ratios. It is worth noting that on the basis of the previous state-of-the-art BCL, our \methodspace further achieves improvements by 1.20$\%$, 1.82$\%$ and 1.22$\%$, respectively. Our \methodspace consistently improves performance across datasets with varying degrees of class imbalance, demonstrating its potential to generalize to practical scenarios with more complex distributions. Specially, our method does not require any prior knowledge or assumptions about the underlying data distribution, highlighting the robustness and versatility to automatically adapt to the data.

(2) \emph{Representation Balancedness.} In \cref{sec:intro}, we claim that \methodspace helps compress the expansion of head classes and avoid the passive collapse of tail classes, yielding the more balanced representation distribution. To justify this aspect, we compare the variance in linear probing performance among many/medium/few groups, namely, their groupwise standard deviation. According to \cref{tab:main}, our \methodspace provides [1.56$\%$, 2.35$\%$, 3.09$\%$], [2.16$\%$, 2.44$\%$, 2.87$\%$] and [1.26$\%$, 1.57$\%$, 1.74$\%$] improvements \textit{w.r.t} [many, medium, few] groups on CIFAR-100-LT-R100/R50/R10 with more preference to the minority classes for representation balancedness. Overall, \methodspace substantially improves the standard deviation by [0.61,~0.38,~0.35] on different levels of imbalance.

\textbf{ImageNet-LT and Places-LT.} \cref{tab:main} shows the comparison of different baseline methods on large-scale dataset ImageNet-LT and Places-LT, in which we have consistent observations. As can be seen, on more challenging real-world dataset, \methodspace still outperforms other methods in terms of overall accuracy, averaging as 1.68$\%$,~0.68$\%$ on ImageNet-LT and Places-LT. Specifically, our method provides [0.44$\%$,~2.23$\%$,~3.25$\%$] and [0.55$\%$,~0.87$\%$,~0.45$\%$] improvements in linear probing \textit{w.r.t.} [many, medium, few] groups on ImageNet-LT and Places-LT. The consistent performance overhead indicates the robustness of our method to deal with long-tailed distribution with different characteristics. Moreover, the averaging improvement of standard deviation is 1.39 on ImageNet-LT, indicating the comprehensive merits of our \methodspace on the minority classes towards the representation balancedness. However, an interesting phenomenon is that the fine-grained performance exhibits a different trend on Places-LT. As can be seen, the performance of head classes is even worse than that of tail classes. The lower performance of the head partition can be attributed to the fact that it contains more challenging classes. 
As a result, we observe that the standard deviation of our \methodspace does not significantly decrease on Places-LT, which requires more effort and exploration for improvement alongside \method.

\begin{table*}[t]
\caption{Supervised long-tailed learning by finetuning on CIFAR-100-LT, ImageNet-LT and Places-LT.  We compare the performance of five self-supervised learning methods as the pre-training stage for downstream supervised logit adjustment~\citep{menon2021long} method. Improv.~($\uparrow$) represents the averaging performance improvements \textit{w.r.t.} different baseline methods. Besides, the performance of logit adjustment via learning from scratch is also reported for comparisons.}
\label{table:downstream}
\resizebox{1\textwidth}{!}{
\begin{tabular}{c|c|cccccccccc|c}
\toprule[1.5pt]
\multirow{2}{*}{Dataset} & \multirow{2}{*}{LA} & \multicolumn{10}{c|}{Logit adjustment pretrained with the following SSL methods} & \multirow{2}{*}{\textbf{Improv.}} \\ \cmidrule[0.6pt]{3-12}
&    & SimCLR  & \multicolumn{1}{c|}{+\method} & Focal & \multicolumn{1}{c|}{+\method} & SDCLR & \multicolumn{1}{c|}{+\method} & DnC & \multicolumn{1}{c|}{+\method} & BCL   & +\method &   \\ \midrule[0.6pt]\midrule[0.6pt]
CIFAR-LT             & 46.61               & 49.81  & \multicolumn{1}{c|}{50.84}    & 49.83 & \multicolumn{1}{c|}{51.04}    & 49.79 & \multicolumn{1}{c|}{50.73} & 49.97 & \multicolumn{1}{c|}{50.84}   & 50.38 & 51.32   & \up{1.00}   \\
ImageNet-LT          & 48.27               & 51.10                   &  \multicolumn{1}{c|}{51.67}    & 51.15 & \multicolumn{1}{c|}{51.82}    & 50.94 & \multicolumn{1}{c|}{51.64} & 51.31 &  \multicolumn{1}{c|}{51.88}  & 51.43 & 52.06    & \up{0.63}   \\ 
Places-LT            &  27.07                   & 32.63                   &  \multicolumn{1}{c|}{33.86}     & 32.69      &   \multicolumn{1}{c|}{33.75}     &  32.55     &    \multicolumn{1}{c|}{34.03}         & 32.98 &  \multicolumn{1}{c|}{34.09}  &  33.15     & 34.48  & \up{1.24}        \\ \bottomrule[1.5pt]
\end{tabular}}

\end{table*}

% \subsection{Downstream Supervised Long-tailed Classification}
\subsection{Downstream Finetuning Evaluation}
\revise{\textbf{Downstream supervised long-tailed learning.}} Self-supervised learning has been proved to be beneficial as a pre-training stage of supervised long-tailed recognition to exclude the explicit bias from the class imbalance~\citep{yang2020rethinking,liu2021self,zhou2022contrastive}. 

\revise{To validate the effectiveness} of our \method, we conduct self-supervised pre-training as the initialization for downstream supervised classification tasks on CIFAR-100-LT-R100, ImageNet-LT and Places-LT. The state-of-the-art logit adjustment~\citep{menon2021long} is chosen as the downstream baseline. The combination of \methodspace + LA can be interpreted as a compounded method where \methodspace aims at the re-balanced representation extraction and LA targets the classifier debiasing. In \cref{table:downstream}, we can find that the superior performance improvements are achieved by self-supervised pre-training over the plain supervised learning baseline. \revise{Besides,} our method can also consistently outperform other SSL baselines, averaging as 1.00$\%$, 0.63$\%$ and 1.24$\%$ on CIFAR-100-LT-R100, ImageNet-LT and Places-LT. These results demonstrate that \methodspace are well designed to facilitate long-tailed representation learning and improve the generalization for downstream supervised tasks.

\revise{\textbf{Cross-dataset transfer learning}. To further demonstrate the representation transferability of our \method, we conduct more comprehensive experiments on the large-scale, long-tailed dataset CC3M~\citep{sharma2018conceptual} with various cross-dataset transferring tasks, including downstream classification, object detection and instance segmentation. Specifically, we report the finetuning classification performance on ImageNet, Places and fine-grained visual datasets Caltech-UCSD Birds (CUB200)~\citep{wah2011caltech}, Aircrafts~\citep{maji2013fine}, Stanford Cars~\citep{krause20133d},  Stanford Dogs~\citep{khosla2011novel}, NABirds~\citep{van2015building}. Besides, we evaluate the quality of the learned representation by finetuning the model for object detection and instance segmentation on COCO2017 benchmark~\citep{lin2014microsoft}. As shown in \cref{table:downstreamclf,table:downstreamdetseg}, we can see that our proposed GH consistently outperforms the baseline across various tasks and datasets. It further demonstrates the importance of considering long-tailed data distribution under large-scale unlabeled data in the pretraining stage. This can potentially be attributed to that our geometric harmonization motivates a more balanced and general emebdding space, improving the generalization ability of the pretrained model to a range of real-world downstream tasks.}

\begin{table}[!t]
\centering
\revise{
\caption{Image classification on ImageNet, Places and fine-grained visual classification on various fine-grained datasets, pretrained on large-scale long-tailed CC3M and then finetuned.}\label{table:downstreamclf}
% \vspace{.4em}
\setlength{\tabcolsep}{3.7pt}
\resizebox{1\textwidth}{!}{
\begin{tabular}{ @{} l cc c cccccc @{} }
\toprule
& \multicolumn{2}{c}{Image Classification} & & \multicolumn{6}{c}{Fine-Grained Visual Classification}\\
\cmidrule{2-3} \cmidrule{5-10}
&ImageNet	&Places	& &CUB200	&Aircraft	&StanfordCars	&StanfordDogs	&NABirds	&Average \\
\midrule
SimCLR &52.06 &37.65 & &44.61 &65.89 &57.63 &50.99 &46.86 &53.20 \\
\rowcolor[HTML]{EFEFEF} 
+GH &\textbf{53.39} &\textbf{38.47} & &\textbf{45.76} &\textbf{68.08} &\textbf{60.24} &\textbf{52.88} &\textbf{47.58} &\textbf{54.91} \\
\bottomrule
\end{tabular}}}
\end{table}
\begin{table}[!t]
\begin{minipage}[t]{0.58\textwidth}
\centering
\revise{
\caption{Object detection and instance segmentation with finetuned features on COCO2017 benchmark, pretrained on large-scale long-tailed CC3M.}\label{table:downstreamdetseg}
\vspace{.4em}
\renewcommand{\arraystretch}{1.2}
\resizebox{1\textwidth}{!}{
\begin{tabular}{ @{} l ccc c ccc @{} }
\toprule
& \multicolumn{3}{c}{Object Detection} & & \multicolumn{3}{c}{Instance Segmentation}\\
\cmidrule{2-4} \cmidrule{6-8}
&AP$^{bbox}$ & AP$^{bbox}_{50}$ & AP$^{bbox}_{75}$ & &AP$^{mask}$ & AP$^{mask}_{50}$ & AP$^{mask}_{75}$ \\
\midrule
SimCLR &31.7 &51.0 &33.9 & &30.2 &49.8 &32.1 \\
\rowcolor[HTML]{EFEFEF} 
+GH &\textbf{32.7} &\textbf{52.2} &\textbf{35.2} & &\textbf{31.1} &\textbf{50.8} &\textbf{33.0} \\
\bottomrule
\end{tabular}}}
\end{minipage}
\begin{minipage}[t]{0.42\textwidth}
\centering
\revise{
\caption{Inter-class uniformity~($\uparrow$) and neighborhood uniformity~($\uparrow$) of pretrained features on CIFAR-LT.}\label{table:uniformity}
\vspace{.4em}
\setlength{\tabcolsep}{3.7pt}
\resizebox{1\textwidth}{!}{
\begin{tabular}{ @{} l cc c cc @{} }
\toprule
& \multicolumn{2}{c}{\small Inter-class Uniformity} & & \multicolumn{2}{c}{\small Neighborhood Uniformity}\\
\cmidrule{2-3} \cmidrule{5-6}
&SimCLR & +GH & &SimCLR & +GH \\
\midrule
C100 &1.00 & \cellcolor{greyL}\textbf{2.80} & & 0.72 & \cellcolor{greyL}\textbf{2.00} \\
C50 &1.23 & \cellcolor{greyL}\textbf{2.73} & & 0.91 & \cellcolor{greyL}\textbf{1.94} \\
C10 &1.18 & \cellcolor{greyL}\textbf{2.60} & & 0.85 & \cellcolor{greyL}\textbf{1.83} \\
\bottomrule
\end{tabular}}}
\end{minipage}
\end{table}

\begin{figure*}[!t]
\centering
% \hspace{10pt}
\subfigure[Geometric Dimension]{
\includegraphics[height=0.22\linewidth]{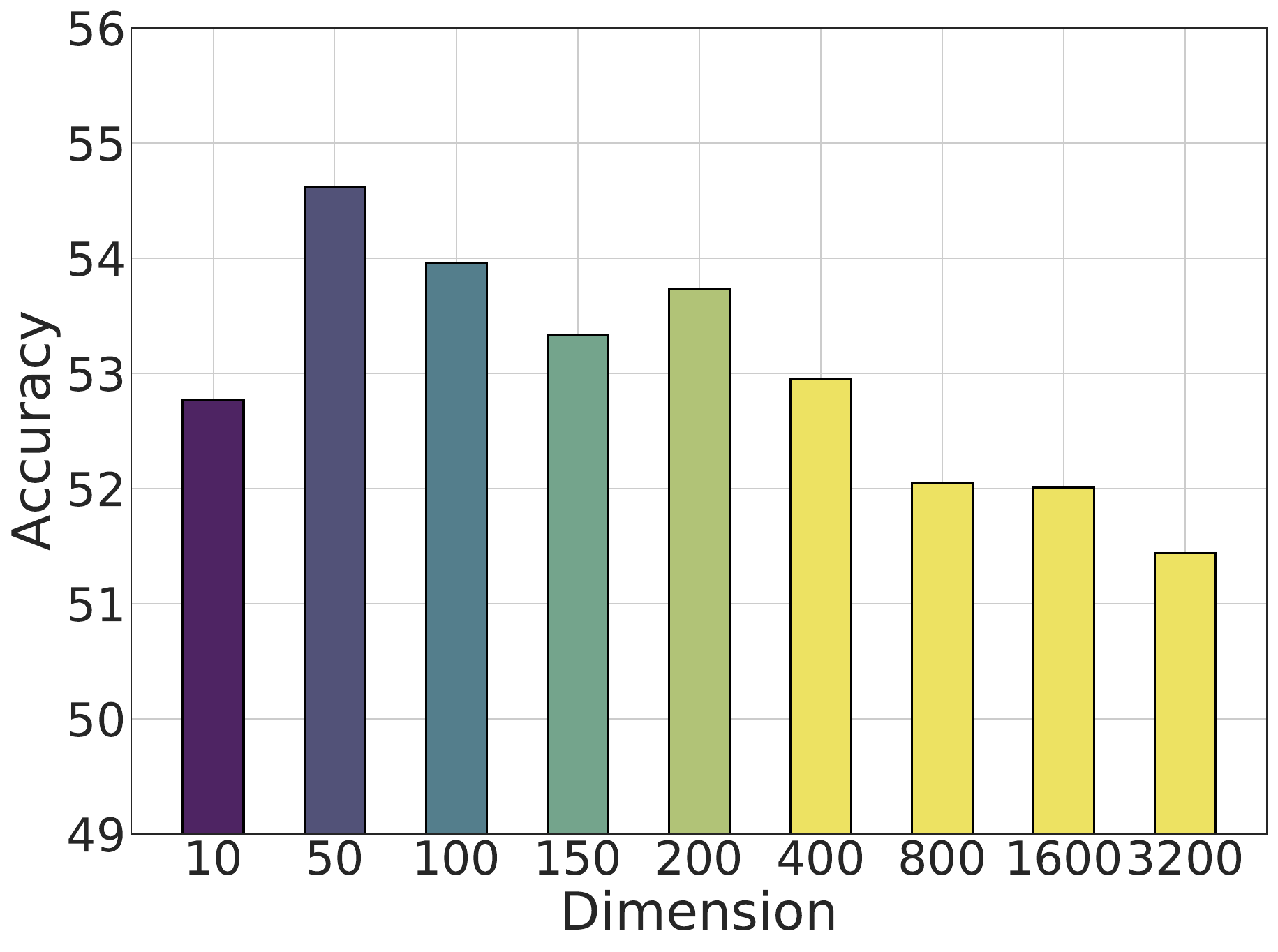}
\label{fig:ablation_dimension}
}
% \hspace{3pt}
\subfigure[Label Quality]{
\includegraphics[height=0.22\textwidth]
{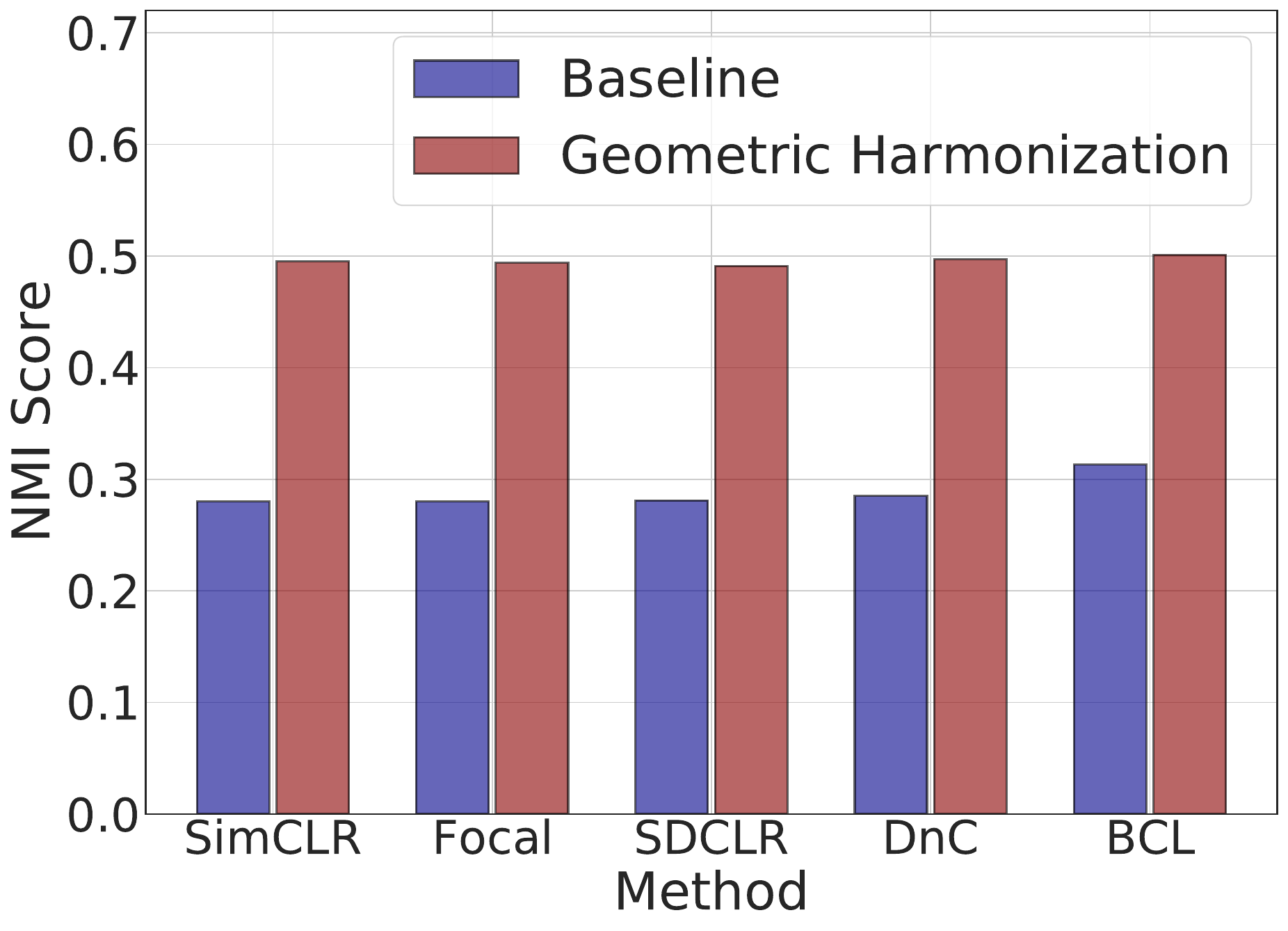}
% {imgs/ablation_nmi_revision_dark.pdf}
\label{fig:ablation_nmi}}
% \hspace{3pt}
\subfigure[Label-Distribution Prior]{
\includegraphics[height=0.22\textwidth]
 {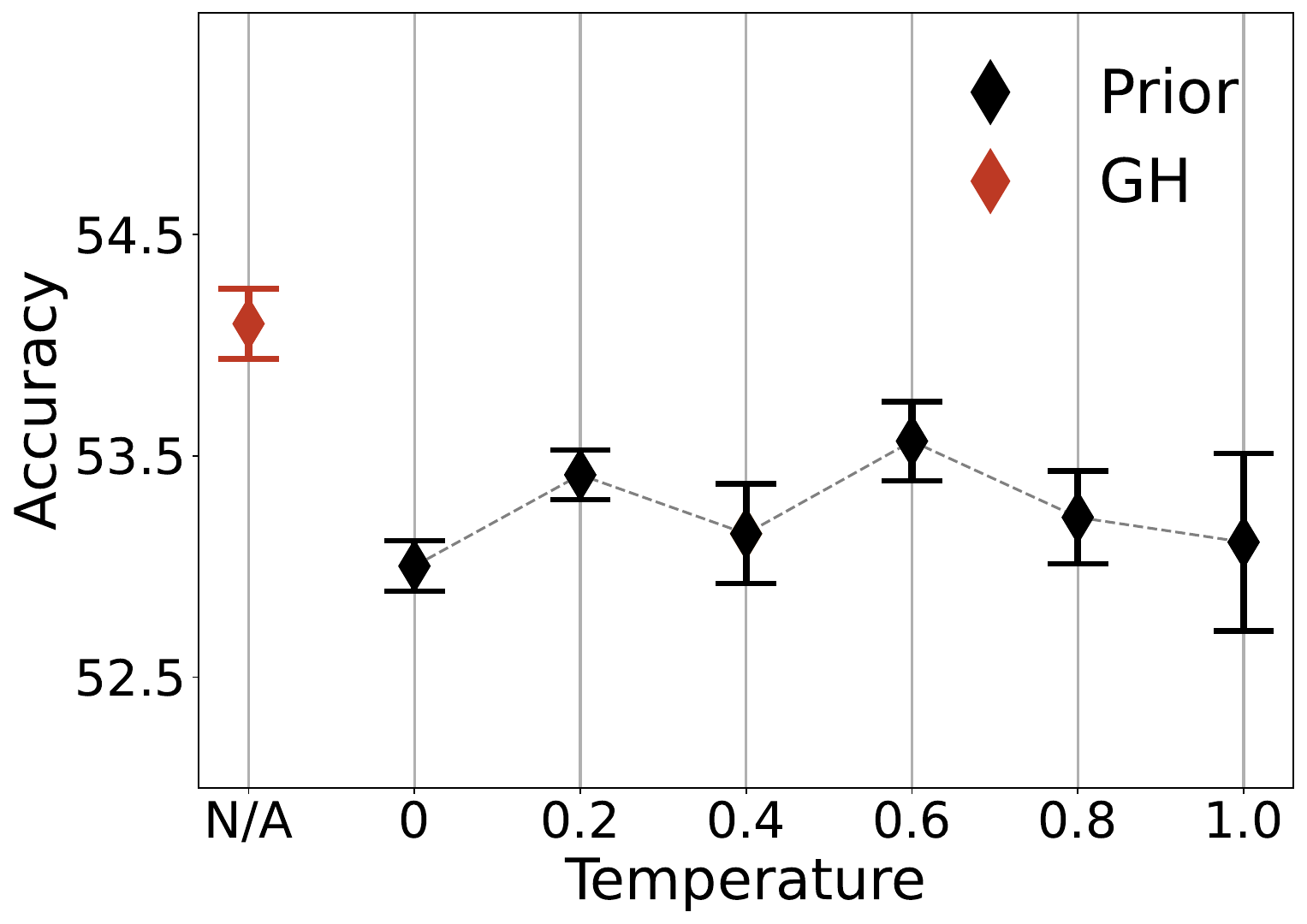}
 % {imgs/ablation_fix_imb_LA0.05_std_0123.pdf}
\label{fig:ablation_fix_imb}
    }
\vspace*{-8pt}
\caption{(a) Linear probing performance \textit{w.r.t.} the dimension $K$ of the geometric uniform structure $\mathbf{M}$~(\cref{appendix:structure}) on CIFAR-LT-R100. (b) NMI score between the surrogate geometric labels and the ground-truth labels in the training stage on CIFAR-LT-R100. (c) Average linear probing  and the error bars of the surrogate label allocation with variants of the label prior on CIFAR-LT-R100.}
\vspace*{-10pt}
\end{figure*}

% \begin{table*}[t]
% % \vspace{-3pt}
% \centering
% \caption{Linear probing of more SSL variants on CIFAR-100-LT with different imbalanced ratios (100,50,10).}\label{tab:moressl}
% \resizebox{\textwidth}{!}{%
% \begin{tabular}{c|ccc|ccc|ccc|ccc}
% \toprule[1.5pt]
% Dataset & SeLa &+\method& \textbf{Improv.} & SwAV &+\method& \textbf{Improv.} & SimSiam &+\method& \textbf{Improv.} & Barlow &+\method& \textbf{Improv.} \\
% \midrule[0.6pt]\midrule[0.6pt]
% CIFAR-LT-R100   & 46.03 &47.15& \up{1.12}                 & 51.40 &52.65& \up{1.23}                 & 49.01 &51.43& \up{2.42}                    & 48.70 &51.23& \up{2.53}                    \\
% CIFAR-LT-R50    & 47.05 &47.95& \up{0.90}                 & 51.70 &53.52& \up{1.82}                 & 48.98 &53.54& \up{4.56}                    & 49.29 &51.95& \up{2.66}                    \\
% CIFAR-LT-R10    & 49.89 &51.63& \up{1.64}                 & 55.27 &56.63& \up{1.36}                 & 55.51 &57.03& \up{1.52}                    & 53.11 &56.34& \up{3.23}          \\\bottomrule[1.5pt]      
% \end{tabular}}
% % \vspace*{-14pt}
% \end{table*}

\begin{table*}[t]
\centering
\revise{
\caption{Linear probing of more SSL variants on CIFAR-100-LT with different imbalanced ratios.}\label{tab:moressl}}
\resizebox{\textwidth}{!}{%
\begin{tabular}{c|cc|cc|cc|cc|cc}
\toprule[1.5pt]
Dataset   & SimSiam &+\method & Barlow &+\method & BYOL &+\method & MoCo-v2 &+\method & MoCo-v3 &+\method \\
\midrule[0.6pt]\midrule[0.6pt]
CIFAR-R100                                     & 49.01 &51.43                   & 48.70 &51.23    & 51.43  & 52.87  & 51.49 &53.53   & 54.08 & 55.82          \\
CIFAR-R50                                    & 48.98 &53.54                    & 49.29 &51.95     & 52.04 & 53.84 & 52.68 &55.01    & 55.34 &  56.45             \\
CIFAR-R10                                     & 55.51 &57.03                    & 53.11 &56.34   & 55.86 & 57.28 & 58.23 & 60.11    & 59.10 & 60.57    \\\bottomrule[1.5pt]      
\end{tabular}}
\vspace*{-15pt}
\end{table*}

\subsection{Further Analysis and Ablation Studies} \label{exp:ablation}

\textbf{Dimension of Geometric Uniform Structure.} As there is even no category number $L$ available in SSL paradigm, we empirically compare our \methodspace with different geometric dimension $K$ on CIFAR-100-LT-R100, as shown in \cref{fig:ablation_dimension}. From the results, \methodspace is generally robust to the change of $K$, but slightly exhibits the performance degeneration when the dimension is extremely large or small. Intuitively, when $K$ is extremely large, our \methodspace might pay more attention to the uniformity among sub-classes, while the desired uniformity on classes is not well guaranteed. Conversely, when $K$ is extremely small, the calibration induced by \methodspace is too coarse that cannot sufficiently avoid the internal collapse within each super-class. For discussions of the structure, it can refer to \cref{appendix:structure}.

\textbf{Surrogate Label Quality Uncovered.} To justify \revise{the effectiveness of surrogate label allocation},  we compare the NMI scores~\citep{strehl2002cluster} between the surrogate and ground-truth label in \cref{fig:ablation_nmi}. We observe that \methodspace significantly improves the NMI scores across baselines, indicating that the geometric labels are effectively calibrated to better capture the latent semantic information. Notably, the improvements of the existing works are marginal, which further verifies the superiority of \method.

\textbf{Exploration with Other Label-Distribution Prior.} To further understand $\boldsymbol{\pi}$, we assume the ground-truth label distribution is available and incorporate the oracle $\boldsymbol{\pi}^{\vy}$ into the surrogate label allocation. Comparing the results of the softened variants $\boldsymbol{\pi}_{\gamma_{\mathrm{T}}}^{\vy}$ with the temperature ${\gamma_{\mathrm{T}}}$ in \cref{fig:ablation_fix_imb}, we observe that \methodspace outperforms all the counterparts equipped with the oracle prior. A possible reason is that our method automatically captures the inherent geometric statistics from the embedding space, which is more reconcilable to the self-supervised learning objectives.

\revise{\textbf{Uniformity Analysis.} In this part, we conduct experiments with two uniformity metrics~\citep{wang2020understanding, li2022targeted}: }
\begin{equation}
    \mathrm{U}=\frac{1}{L(L-1)}\sum_{i=1}^L\sum_{j=1, j\neq i}^L||\boldsymbol{\mu}_i-\boldsymbol{\mu}_j||_2, \ \ \  \mathrm{U}_k=\frac{1}{Lk}\sum_{i=1}^L\min_{j_1,\dots, j_k}(\sum_{m=1}^k||\boldsymbol{\mu}_i-\boldsymbol{\mu}_{j_m}||_2), \nonumber
\end{equation}
\revise{where $j_1, \dots, j_k \neq i$ represent different classes. Specifically, $\mathrm{U}$ evaluates average distances between different class centers and $\mathrm{U}_k$ measures how close one class is to its neighbors. As shown in \cref{table:uniformity}, our \methodspace outperforms in both inter-class uniformity and neighborhood uniformity when compared with the baseline SimCLR~\citep{chen2020simple}. This indicates that vanilla contrastive learning struggles to achieve the uniform partitioning of the embedding space, while our \methodspace effectively mitigates this issue.}

\textbf{Comparison with More SSL Methods.} In \cref{tab:moressl},~we present a more comprehensive comparison of different SSL baseline methods, \revise{including MoCo-v2~\cite{he2020momentum}, MoCo-v3~\cite{chen2021mocov3} and various non-contrastive methods such as SimSiam~\cite{chen2021exploring}, BYOL~\cite{grill2020bootstrap} and Barlow~Twins~\cite{zbontar2021barlow}.} From the results, we can see that the combinations of different SSL methods and our \methodspace can achieve consistent performance improvements, averaging as 2.33$\%$, 3.18$\%$ and 2.21$\%$ on CIFAR-100-LT. This demonstrates the prevalence of representation learning disparity under data imbalance in general SSL settings.

\begin{wraptable}{r}{0.46\linewidth}
\centering
% \vspace*{-16pt}
\vspace*{-3pt}
\caption{Linear probing of joint optimization on CIFAR-100-LT with different IRs.} \label{tab:joint}
\vspace{2pt}
\resizebox{0.43\textwidth}{!}{
\begin{tabular}{c|ccc}
\toprule[1.5pt]
 Method & C100 & C50 & C10 \\
\midrule
SimCLR & 50.72 & 52.24 & 55.67 \\\midrule
+\methodspace (Joint) & 50.18 & 52.31 & 54.98 \\
w/ warm-up & 51.14 & 52.75 & 55.37 \\\midrule
+\methodspace (Bi-level) & 53.96 & 55.42 & 57.36 \\
\bottomrule[1.5pt]
\end{tabular}
}\vspace*{-10pt}
\end{wraptable}

\textbf{On Importance of Bi-Level Optimization.} 
In \cref{tab:joint}, we empirically compare the direct joint optimization strategy to \eqref{eq:overall}. From the results, we can see that the joint optimization (\textit{w/} or \textit{w/o} the warm-up strategy) does not bring significant performance improvement over SimCLR compared with that of our bi-level optimization, probably due to the undesired collapse in label allocation~\cite{asano2020self}. This demonstrates the necessity of the proposed bi-level optimization for~\eqref{eq:overall} to stabilize the training. 

\revise{\textbf{Qualitative Visualization.} We conduct t-SNE visualization of the learned features to provide further qualitative intuitions. For simplity, we randomly selected four head classes and four tail classes on CIFAR-LT to generate the t-SNE plots. Based on the results in \cref{fig:appendix_tsne}, the observations are as follows: (1) SimCLR: head classes exhibit a large presence in the embedding space and heavily squeeze the tail classes, (2) GH: head classes reduce their occupancy, allowing the tail classes to have more space. This further indicates that the constructed surrogate labels can serve as the high-quality supervision, effectively guiding the harmonization towards the geometric uniform structure.}

\revise{\textbf{Sensitivity Analysis.} To further validate the stability of our \method, We conduct empirical comparison with different weight $w_{\mathrm{\method}}$, temperature $\gamma_{\mathrm{\method}}$, regularization coefficient $\lambda$ and Sinkhorn iteration $E_s$ on CIFAR-LT, as shown in \cref{fig:weight,fig:appendix_ablation}. From the results, we can see that our \methodspace can consistently achieve satisfying performance with different hyper-parameter.}

\begin{figure}[!t]
    \begin{minipage}[H]{0.68\textwidth}
        \centering
        \subfigure{
        \begin{minipage}[H]{0.44\textwidth}
        \centering
        \includegraphics[width=1\linewidth]{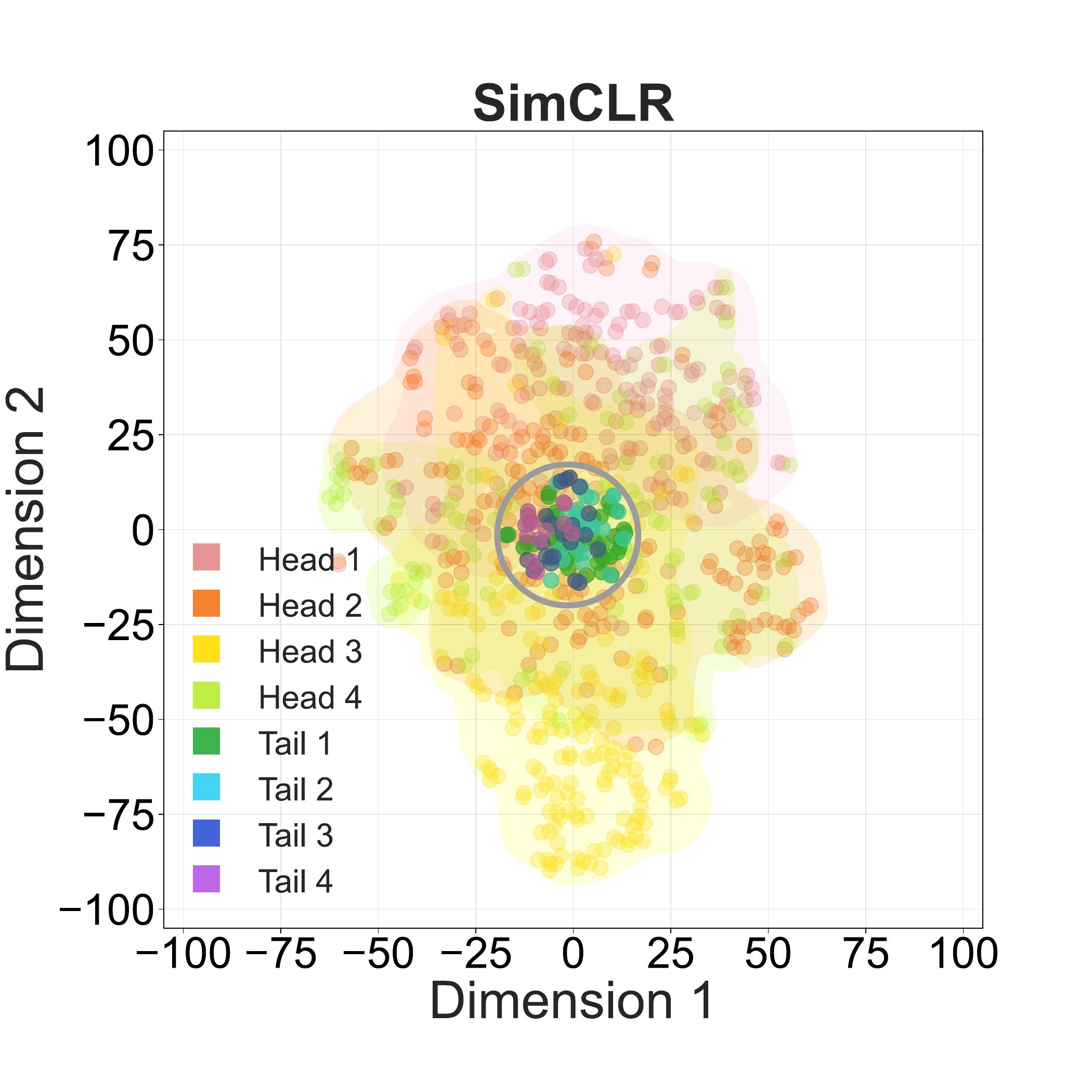}
        \end{minipage}
        }
        \subfigure{
        \begin{minipage}[H]{0.44\textwidth}
        \centering
        \includegraphics[width=1\linewidth]{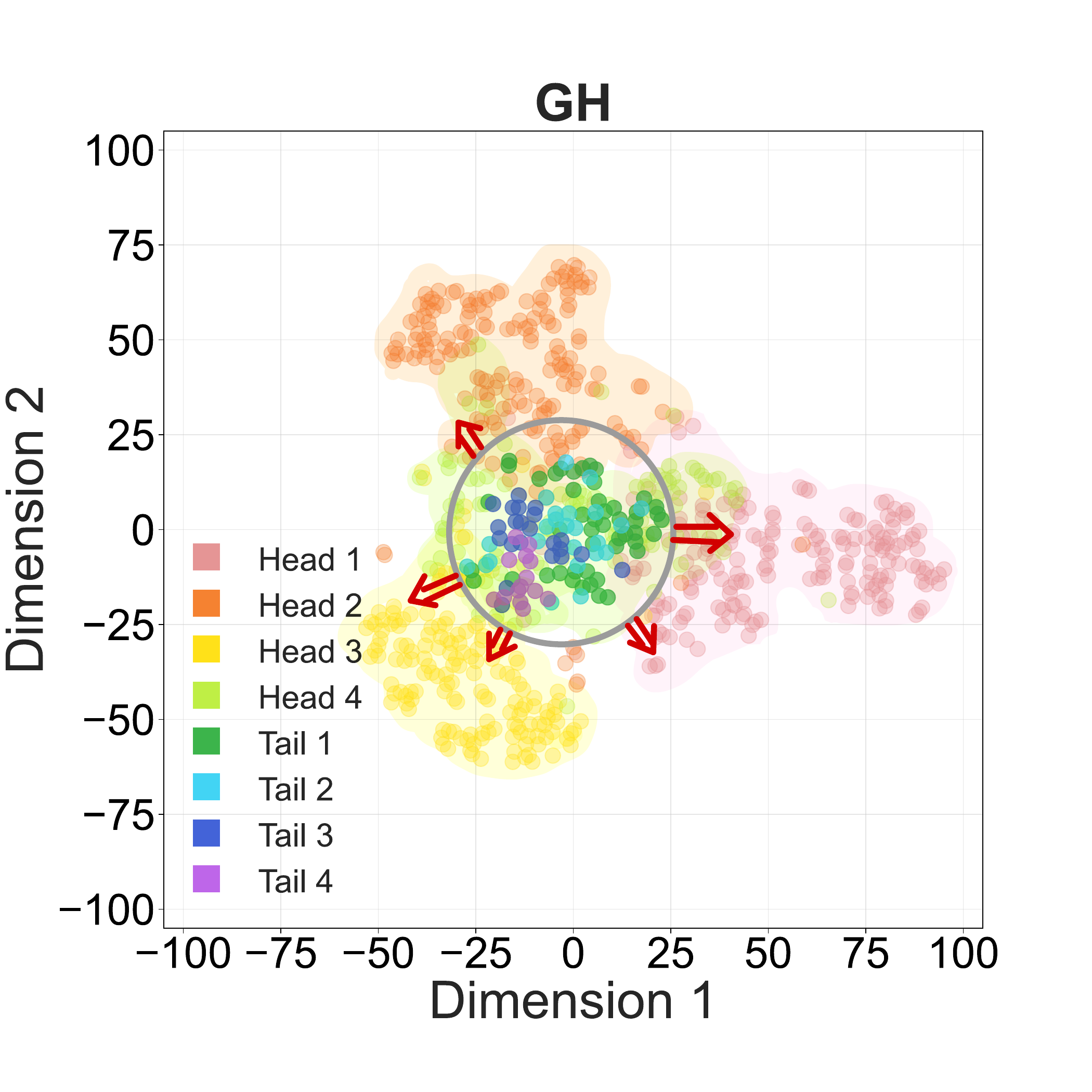}
        \end{minipage}
        }
        % \vspace{-20pt}
        \vspace{-15pt}
        \caption{T-SNE visualization of the pretrained features from 8 randomly selected classes \textit{w.r.t} CIFAR-LT training images.}
        \label{fig:appendix_tsne}
    \end{minipage}
    \begin{minipage}[H]{0.25\textwidth}
        \vspace{15pt}
        \centering
        \includegraphics[width=1.15\linewidth]{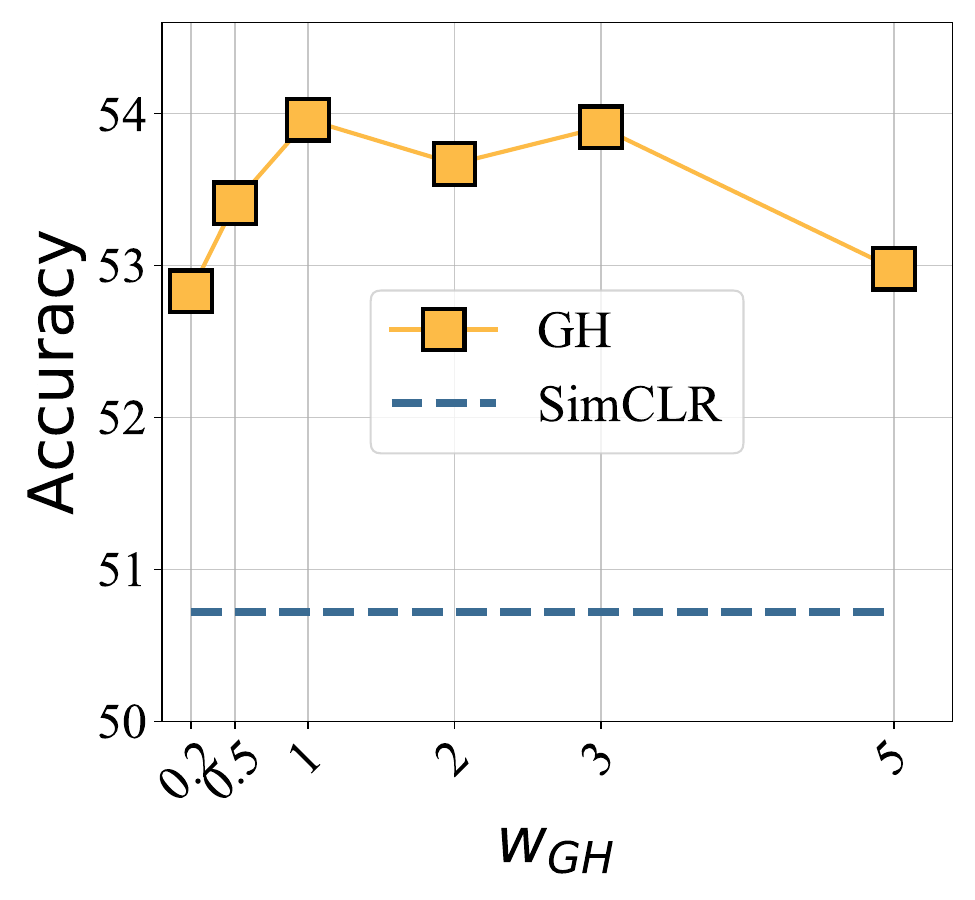}
        \vspace{-20pt}
        % \vspace{-25pt}
        \caption{Ablations of $w_{GH}$ on CIFAR-LT.}\label{fig:weight}
    \end{minipage}
\end{figure}

\section{Conclusion}
In this paper, we delve into the defects of the conventional contrastive learning in self-supervised long-tail context, \textit{i.e.}, representation learning disparity, motivating our exploration on the inherent intuition for approaching the category-level uniformity. From the geometric perspective, we propose a novel and efficient Geometric Harmonization algorithm to counteract the long-tailed effect on the embedding space, \textit{i.e}, over expansion of the majority class with the passive collapse of the minority class. Specially, our proposed \methodspace leverages the geometric uniform structure as an optimal indicator and manipulate a fine-grained label allocation to rectify the distorted embedding space. We theoretically show that our proposed method can harmonize the desired geometric property in the limit of loss minimum. It is also worth noting that our method is orthogonal to existing self-supervised long-tailed methods and can be easily plugged into these methods in a lightweight manner. Extensive experiments demonstrate the consistent efficacy and robustness of our proposed \method. We believe that the geometric perspective has the great potential to evolve the general self-supervised learning paradigm, especially when coping with the class-imbalanced scenarios.

\section*{Acknowledgement}

This work was supported by the National Key R\&D Program of China (No. 2022ZD0160702),  STCSM (No. 22511106101, No. 22511105700, No. 21DZ1100100), 111 plan (No. BP0719010) and National Natural Science Foundation of China (No. 62306178). BH was supported by the NSFC Young Scientists Fund No. 62006202, NSFC General Program No. 62376235, and Guangdong Basic and Applied Basic Research Foundation No. 2022A1515011652.

\bibliography{main}
\bibliographystyle{plainnat}

%%%%%%%%%%%%%%%%%%%%%%%%%%%%%%%%%%%%%%%%%%%%%%%%%%%%%%%%%%%%%%%%%%%%%%%%%%%%%%%
%%%%%%%%%%%%%%%%%%%%%%%%%%%%%%%%%%%%%%%%%%%%%%%%%%%%%%%%%%%%%%%%%%%%%%%%%%%%%%%
% APPENDIX
%%%%%%%%%%%%%%%%%%%%%%%%%%%%%%%%%%%%%%%%%%%%%%%%%%%%%%%%%%%%%%%%%%%%%%%%%%%%%%%
%%%%%%%%%%%%%%%%%%%%%%%%%%%%%%%%%%%%%%%%%%%%%%%%%%%%%%%%%%%%%%%%%%%%%%%%%%%%%%%
\newpage
\appendix
\onecolumn

\vbox{%
    \hsize\textwidth
    \linewidth\hsize
    \vskip 0.1in
    \rule{\textwidth}{3.5pt}
  	\vskip 0.1in
    \centering
    {\LARGE\bf Appendix: Combating Representation Learning Disparity with Geometric Harmonization \par}
    \vskip 0.1in
%    \hrule
	\rule{\textwidth}{1pt}
  }

\section*{Contents}
\startcontents[chapters]
\printcontents[chapters]{}{1}{\contentsmargin{1em}}

\section*{Reproducibility Statement}\addcontentsline{toc}{section}{Reproducibility Statement}

We provide our source codes to ensure the
reproducibility of our experimental results. Below we summarize several critical aspects \textit{w.r.t} the reproducible results:

\begin{itemize}
    \item \textbf{Datasets.} The datasets we used are all publicly accessible, which is introduced in \cref{appendix:dataset}. For long-tailed subsets, we strictly follows previous work~\cite{kang2019decoupling} on CIFAR-100-LT to avoid the bias attribute to the sampling randomness. On ImageNet-LT and Places-LT, we employ the widely-used data split first introduced in \cite{liu2019large}.
    \item \textbf{Source code.} Our code is available at \href{https://github.com/MediaBrain-SJTU/Geometric-Harmonization}{https://github.com/MediaBrain-SJTU/Geometric-Harmonization}.
    \item \textbf{Environment.} All the experiments are conducted on NVIDIA GeForce RTX 3090 with Python 3.7 and Pytorch 1.7.
\end{itemize}

\section{Additional Discussions of Related Works}\label{appendix:related-work}

\subsection{Supervised Long-tailed Learning}
% \textbf{Supervised Long-tailed Learning.}
 As the explorations on the classifier learning~\citep{kang2019decoupling,yao2023latent} are orthogonal to the self-supervised learning paradigms, we mainly focus on the representation learning in supervised long-tailed recognition. The pioneering work~\citep{kang2019decoupling} first explores representation and classifier learning with a disentangling mechanisms and shows the merits of instance-balanced sampling strategy on the representation learning stage. Subsequently, \citet{yang2020rethinking} points out the negative impact of label information and proposes to improve the representation learning with semi-supervised learning and self-supervised learning. This motivates a stream of research works diving into the representation learning. Supervised contrastive learning~\citep{kang2020exploring, cui2021parametric} is leveraged with rebalanced sampling or prototypical learning design to pursue a more balanced representation space. \citet{li2022targeted} explicitly regularizes the class centers to a maximum separation structure with similar drives to the balanced feature space.

 \subsection{Contrastive Learning is Still Vulnerable to Long-tailed Distribution}

 The prior works~\citep{kang2020exploring,liu2021self} point out that contrastive learning can extract more balanced features compared with the supervised learning paradigm. However, several subsequent works~\citep{jiang2021self, zhou2022contrastive} empirically observes that contrastive learning is still vulnerable to the long-tailed distribution, which motivates their model-pruning strategy~\citep{jiang2021self} and memorization-oriented augmentation~\citep{zhou2022contrastive} to rebalance the representation learning. In this paper, we delve into the intrinsic limitation of the contrastive learning method in the long-tailed context, \textit{i.e}, approaching sample-level uniformity to deteriorate the embedding space.

\subsection{Unsupervised Clustering}

Deep Cluster~\citep{caron2018deep} applies K-Means clustering to generate pseudo-labels for the unlabeled data, which are then iteratively leveraged as the supervised signal to train a classifier. SeLa~\citep{asano2020self} first casts the pseudo-label generation as an optimal transport problem and leverages a uniform prior to guide the clustering. SwAV~\citep{caron2020unsupervised} adopts mini-batch clustering instead of dataset-level clustering, enhancing the practical applicability of the optimal transport-based clustering method. Subsequently, \citet{li2021prototypical} combines clustering and contrastive learning objectives in an Expectation-Maximization framework, recursively updating the data features towards their corresponding class prototypes. In this paper, we propose a novel Geometric Harmonization method that is capable to cope with long-tailed distribution, the uniqueness can be summarized in the following aspects: (1) \emph{Geometric Uniform Structure}. The pioneering works~\citep{asano2020self,caron2020unsupervised} mainly resort to a learnable classifier to perform clustering, which can easily be distorted in the long-tailed scenarios~\citep{fang2021exploring}. Built on the geometric uniform structure, our method is capable to provide high-quality clustering results with clear geometric interpretations. (2) \emph{Flexible Class Prior}. The class prior in ~\citep{asano2020self,caron2020unsupervised} is assumed to be uniform among the previous attempts. When moving to the long-tailed case, this assumption will strengthen the undesired sample-level uniformity. In contrast, our methods can potentially cope with any distribution with the automatic surrogate label allocation. (3) \emph{Theoretical Guarantee.} \methodspace is theoretically grounded to achieve the category-level uniformity in the long-tailed scenarios, which has never been studied in previous methods.

\subsection{Taxonomy of Self-supervised Long-tailed Methods}

We summarize the detailed taxonomy of self-supervised long-tailed methods in \cref{algorithm:allocation}.

\begin{table}[!htb]
\centering
\revise{
\caption{Taxonomy of self-supervised long-tailed methods.}
\label{tab:taxonomy}
\begin{tabular}{lcc}
\toprule[1.5pt]
Method &	Aspect &	Description                                                    \\ \midrule
Focal~\citep{lin2017focal} & Sample Reweighting & Hard example mining \\ 
rwSAM~\citep{liu2021self} & Optimization Surface & Data-dependent sharpness-aware minimization \\ 
SDCLR~\citep{jiang2021self} & Model Pruning & Model pruning and self-contrast \\ 
DnC~\citep{tian2021divide} & Model Capacity & Multi-expert ensemble \\ 
BCL~\citep{zhou2022contrastive} & Data Augmentation & Memorization-guided augmentation \\ 
GH & Loss Limitation & Geometric harmonization \\ \bottomrule[1.5pt]
\end{tabular}}
\end{table}

\section{Discussions of Geometric Uniform Structure~(\cref{def:structure})}\label{appendix:structure}

\subsection{Simplex Equiangular Tight Frame~($K\leq d$)}
Neural collapse~\citep{mixon2022neural} describes a phenomenon that with the training, the geometric centroid of representation progressively collapses to the optimal classifier parameter \textit{w.r.t.} each category. The collection of these points builds a special geometric structure, termed as Simplex Equiangular Tight Frame (ETF). Some study that shares the similar spirit is also explored regarding the maximum separation structure~\citep{kasarla2022maximum}. We present its formal definition as follows.
\begin{definition}\label{def:etf}
A Simplex ETF is a collection of points in $\mathbb{R}^d$ specified by the columns of the matrix:
\begin{equation}
\label{eq:def}
    \mathbf{M}^{\mathrm{ETF}} = \sqrt{\frac{K}{K-1}} \mathbf{U}(\mathbf{I}_K - \frac{1}{K}\mathbbm{1}_K\mathbbm{1}^{\mathrm{T}}_K),
\end{equation}
\end{definition}
where $\mathbf{I}_K \in \mathbb{R}^{K \times K}$ is the identity matrix and $\mathbbm{1}_K$ is the $K$-dimensional ones vector. $\mathbf{U} \in \mathbb{R}^{d \times K}$ is the patial orthogonal matrix such that $\mathbf{U}^\top \mathbf{U} = \mathbf{I}_K$ and it satisfys $d \geq K$. All vectors in a Simplex ETF have the same pair-wise angle, \textit{i.e.}, $\mathbf{M}^{\mathrm{ETF}}_i \mathbf{M}^{\mathrm{ETF}}_j=-\frac{1}{K-1}, 1 \leq i \neq j \leq K$. The pioneering work~\citep{yang2022we} shows Simplex ETF as a linear classifier combined with neural networks is robust to class-imbalanced learning in the supervised setting.
On the opposite, our motivation is to make self-supervised learning robust to the class-imbalance data, which requires the pursuit in the embedding space intrinsically switching from the sample-level uniformity to the category-level uniformity. The Simplex ETF is a tool to measure the gap between the category-level uniformity and the sample-level uniformity, which is then transformed as the supervision feedback to the training.

\subsection{Alternative Uniform Structure~($K > d$)}

For Simplex ETF, there is a hard dimension constraint in \eqref{eq:def}, \textit{i.e.}, $K\leq d$. However, if this constraint violates, we do not have such a structure in the hyperspherical space. Alternatively, we can conduct the gradient descent to find an approximation of the maximum separation vertices applied into \method. This refers to minimising the following loss function as demonstrated in~\citep{li2022targeted}.

\begin{equation} \label{eq:approx}
    \mathcal{L}_{\mathrm{AP}} = \log \sum_{i=1}^{K} \sum_{j=1}^{K}e^{\tilde{\mathbf{M}}_i \cdot \tilde{\mathbf{M}}_j/\tau_{\mathrm{u}}}, \ \ \ \  \mathrm{s.t.} \ \sum_{i=1}^{K}\tilde{\mathbf{M}}_i=0 \ \text{~and~} \ \forall i \in K~\ \Vert\tilde{\mathbf{M}}_i\Vert=1, 
\end{equation}

where the loss term penalizes the pairwise similarity of different vertices~\citep{wang2020understanding}.

\subsection{Choosing Implementations According to the Dimensional Constraints}

As mentioned above, computing the geometric uniform structure $\mathbf{M}$ becomes much harder in the regime of the limited dimension~($K>d$) regarding the hypersphere space~\cite{graf2021dissecting}. To mitigate this issue, we provide both analytical and approximate solutions for adapting to different application scenarios. Concretely, we choose Simplex ETF~(\cref{def:etf}) when $K\leq d$ or the approximated alternatives~(\eqref{eq:approx}) otherwise. More experimental results can be referred to \cref{sec:impgus}.

% \newpage

\section{Theoretical Proofs and Discussions}\label{appendix:proofs}

\subsection{Warmup}

We begin by introducing the following lower bound~\citep{wang2021chaos} for analyzing the InfoNCE loss.

\begin{lemma}\label{lemma:lower}
(Lower bound for InfoNCE loss). Assume the labels are one-hot and consistent between positive samples: $\forall \vx, \vx^{+} \in p(\vx,\vx^{+}), p(\vy|\vx)=p(\vy|\vx^{+})$. Let $\mathcal{L}_{CE}^{\boldsymbol{\mu}}(f)  =  \E_{p(x,y)} \left[-\log \frac{\exp \left(f(\vx)^\top \boldsymbol{\mu}_\vy\right)}{\sum^{K}_{i=1}\exp\left(f(\vx)^\top \boldsymbol{\mu}_i\right)}\right]$ denote the mean CE loss. For $\forall f \in \mathcal{F}$, the contrastive learning risk $\mathcal{L}_{\mathrm{InfoNCE}}(f,\vx,\vx^{+})$ can be bounded by the classification risk $\mathcal{L}^{\mu}_{\mathrm{CE}}(f,\vx)$,

\begin{equation}
    \mathcal{L}_{\mathrm{InfoNCE}}(f) 
    \geq \mathcal{L}^{\mu}_{\mathrm{CE}}(f) - \sqrt{\mathrm{Var}\left(f(\vx)|\vy\right)}-\mathcal{O}\left(J^{-\frac{1}{2}}\right) + \log\left(\frac{J}{L}\right)
\end{equation}

where $\sqrt{\mathrm{Var}\left(f(\vx)|\vy\right)}$ denotes the conditional intra-class variance $\E_{p(\vy)}\left[\E_{p(\vx|\vy)}\Vert f(\vx)-\E_{p(\vx|\vy)}f(\vx) \Vert^2\right]$, $\mathcal{O}\left(J^{-\frac{1}{2}}\right)$ denotes the Monte Carlo sampling error with $J$ samples and $\log\left(\frac{J}{L}\right)$ is a constant.
\end{lemma}

\def\gL{{\mathcal{L}}}

\begin{proof}
Let $p(\vx,\vx^+,\vy)$ denote the joint distribution $\vx,\vx^+$ with the label $\vy,\ \vy=1,\dots,L$. Denote the negative sample collections as $\{\vx_i^-\}_{i=1}^J$. According to above assumption on label consistency between positive pairs, we have $\vx^+$ and $\vx$ with the same label $\vy$. Denote $\boldsymbol{\mu}_{\vy}$ the class means of class $\vy$ in the embedding space. Then we have the following lower bounds of the InfoNCE loss,

\begin{footnotesize}
\begin{align*}
&\gL_{\rm NCE}(f)=-\E_{p(\vx,\vx^+)}f(\vx)^\top f(\vx^+)+\E_{p(\vx)}\E_{p(\vx^-_i)}\log\sum_{i=1}^J\exp(f(\vx)^\top f(\vx^-_i))\\
=&-\E_{p(\vx,\vx^+)}f(\vx)^\top f(\vx^+)+\E_{p(\vx)}\E_{p(\vx^-_i)}\log\frac{1}{J}\sum_{i=1}^J\exp(f(\vx)^\top f(\vx^-_i))+\log J\\
\stackrel{(1)}{\geq} & -\E_{p(\vx,\vx^+)}f(\vx)^\top f(\vx^+)+\E_{p(\vx)}\log\frac{1}{J}\E_{p(\vx_i^-)}\sum_{i=1}^J\exp(f(\vx)^\top f(\vx^-_i))- A(J)+\log J\\
=&-\E_{p(\vx,\vx^+)}f(\vx)^\top f(\vx^+)+\E_{p(\vx)}\log\E_{p(\vx^-)}\exp(f(\vx)^\top f(\vx^-))- A(J) + \log J\\
=& -\E_{p(\vx,\vx^+,\vy)}f(\vx)^\top f(\vx^+)+\E_{p(\vx)}\log\E_{p(\vy^-)}\E_{p(\vx^-|\vy^-)}\exp(f(\vx)^\top f(\vx^-))- A(J)+\log J\\
\stackrel{(2)}{\geq}& -\E_{p(\vx,\vx^+,\vy)}f(\vx)^\top f(\vx^+)+\E_{p(\vx)}\log\E_{p(\vy^-)}\exp(
\E_{p(\vx^-|\vy^-)}\left[f(\vx)^\top f(\vx^-)\right]) - A(J)+\log J\\
=&-\E_{p(\vx,\vx^+,\vy)}f(\vx)^\top({\boldsymbol{\mu}}_{\vy} + f(\vx^+)-{\boldsymbol{\mu}}_\vy) +\E_{p(\vx)}\log\E_{p(\vy^-)}\exp(\E_{p(\vx^-|\vy^-)}\left[f(\vx)^\top f(\vx^-)\right])- A(J)+\log J\\
{=} &-\E_{p(\vx,\vx^+,\vy)}[f(\vx)^\top{\boldsymbol{\mu}}_{\vy} + f(\vx)^\top(f(\vx^+)-{\boldsymbol{\mu}}_\vy)] +\E_{p(\vx)}\log\E_{p(\vy^-)}\exp(f(\vx)^\top{\boldsymbol{\mu}}_{\vy^-})- A(J)+\log J\\
\stackrel{(3)}{\geq} &-\E_{p(\vx,\vx^+,\vy)}\left[f(\vx)^\top{\boldsymbol{\mu}}_{\vy} + \Vert(f(\vx^+)-{\boldsymbol{\mu}}_\vy)\Vert\right] +\E_{p(\vx)}\log\E_{p(\vy^-)}\exp(f(\vx)^\top{\boldsymbol{\mu}}_{\vy^-})- A(J)+\log J\\
\stackrel{(4)}{\geq} &-\E_{p(\vx,\vy)}f(\vx)^\top{\boldsymbol{\mu}}_{\vy} -\sqrt{\E_{p(\vx,\vy)}\Vert f(\vx)-{\boldsymbol{\mu}}_\vy\Vert^2} +\E_{p(\vx)}\log\E_{p(\vy^-)}\exp(f(\vx)^\top{\boldsymbol{\mu}}_{\vy^-})- A(J)+\log J\\
=&-\E_{p(\vx,\vy)}f(\vx)^\top{\boldsymbol{\mu}}_\vy -\sqrt{\var( f(\vx)\mid \vy)} +\E_{p(\vx)} \log\frac{1}{L}\sum_{k=1}^L\exp(f(\vx)^\top{\boldsymbol{\mu}}_k) - A(J) +\log J \\
=&\E_{p(\vx,\vy)}\big[-f(\vx)^\top{\boldsymbol{\mu}}_\vy+\log\sum_{k=1}^L\exp(f(\vx)^\top{\boldsymbol{\mu}}_k)\big]-\sqrt{\var( f(\vx)\mid \vy)} - A(J) +\log(J/L) \\
=&\gL^{\boldsymbol{\mu}}_{\rm CE}(f) -\sqrt{\var( f(\vx)\mid \vy)} - A(J) +\log(J/L),
% \geq & \gL_{\rm CE}(f) -|f(\vx)^\top\sigma_\vy|- A(M) +\log \frac{M}{K}
\end{align*}
 where (1) follows Lemma \ref{lemma:lse-monte-carlo}; (2) follows the Jensen's inequality for the convex function $\exp(\cdot)$; (3) follows the hyperspherical distribution $f(\vx)\in\sS^{m-1}$, we have
\begin{equation}
    f(\vx)^\top(f(\vx^+)-{\boldsymbol{\mu}}_\vy)\leq \left(\frac{f(\vx^+)-{\boldsymbol{\mu}}_\vy}{\Vert f(\vx^+)-{\boldsymbol{\mu}}_\vy\Vert}\right)^\top(f(\vx^+)-{\boldsymbol{\mu}}_\vy)=\Vert f(\vx^+)-{\boldsymbol{\mu}}_\vy\Vert;
\end{equation}
and (4) follows the Cauchy–Schwarz inequality and the fact that as $p(\vx,\vx^+)=p(\vx^+,\vx)$ holds, $\vx,\vx^+$ have the same marginal distribution. 
% From the lemma above, we know that the InfoNCE loss is an upper bound of the (mean) CE loss asymptotically (as $M\to\infty)$. 
\end{footnotesize}
\end{proof}

In the above proof, the approximation error of the Monte Carlo estimate~\cite{wang2021chaos} can be referred to the following lemma. 

\begin{lemma} (Upper bound of the approximation error by Monte Carlo estimate)
For ${\rm LSE}:=\log\E_{p(\vz)}\exp (f(\vx)^\top g(\vz))$, we denote its (biased) Monte Carlo estimate with $J$ random samples $\vz_i\sim p(\vz),i=1,\dots,J$ as $\widehat{\rm LSE}_J=\log\frac{1}{J}\sum_{i=1}^J \exp (f(\vx)^\top g(\vz_i))$. Then the approximation error $A(J)$ can be upper bounded in expectation as 
\begin{equation}
A(J):=\E_{p(\vx,\vz_i)}|\widehat{\rm LSE}(J) - {\rm LSE}|\leq\gO(J^{-1/2}).
\end{equation}
We can see that the approximation error converges to zero in the order of $1/J^{-1/2}$. 
\label{lemma:lse-monte-carlo}
\end{lemma}

Now we analyze the conditions of \cref{lemma:lower} to strictly achieve its lower bound. In the proof of \cref{lemma:lower}, we have four inequality cases and discuss each one as follows:

(1) According to \cref{lemma:lse-monte-carlo}, we can have the approximation error converges to zero~($A(J)\rightarrow0$) as the sample population increases to the positive infinity~($J\rightarrow+\infty$). Considering the substantial data amount with regard to the benchmark datasets nowadays, we assume $J$ is large enough and the approximation error can achieve zeros, \textit{i.e.}, $A(J)=0$.

(2) follows the Jensen’s inequality as
\begin{equation}
    \E_{p(\vx)}\log\E_{p(\vy^-)}\E_{p(\vx^-|\vy^-)}\exp(f(\vx)^\top f(\vx^-)) \geq \E_{p(\vx)}\log\E_{p(\vy^-)}\exp(
\E_{p(\vx^-|\vy^-)}\left[f(\vx)^\top f(\vx^-)\right]). 
\end{equation}
The equality requires the $\exp(\cdot)$ term as a constant:
\begin{equation}\label{condition2}
    \E_{p(\vx)}\E_{p(\vx^-)} \exp(f(\vx)^\top f(\vx^-)) \equiv C_{(2)} 
\end{equation}
% $\exp(f(x)^\top f(x^-)$
(3) The inequality follows
\begin{equation}
    f(\vx)^\top(f(\vx^+)-\boldsymbol{\mu}_{\vy})\leq \left(\frac{f(\vx^+)-\boldsymbol{\mu}_{\vy}}{\Vert f(\vx^+)-\boldsymbol{\mu}_{\vy}\Vert}\right)^\top(f(\vx^+)-\boldsymbol{\mu}_{\vy})=\Vert f(\vx^+)-\boldsymbol{\mu}_{\vy}\Vert;
\end{equation}
where the equality requires $f(\vx)$ has the same direction with $f(\vx^+)-\boldsymbol{\mu}_{\vy}$. Considering the case 
\begin{equation}\label{condition3}
    \E_{p(\vx,\vx^+,y)}\left[f(\vx^+)-\boldsymbol{\mu}_{\vy}\right]\equiv0, 
\end{equation}
we should have $\E_{p(\vx,\vx^+,y)}\left[f(\vx)^\top(f(\vx^+)-\boldsymbol{\mu}_{\vy})\right]\equiv0$, so the inequality can be simply eliminated from the proof.

(4) Similar in (3), we can simply remove the term $\Vert(f(\vx^+)-\boldsymbol{\mu}_{\vy})\Vert$ in $\E_{p(\vx,\vx^+,y)}\left[f(\vx)^\top\boldsymbol{\mu}_{\vy} + \Vert(f(\vx^+)-\boldsymbol{\mu}_{\vy})\Vert\right]$ when $\E_{p(\vx,\vx^+,y)}\left[f(\vx^+)-\boldsymbol{\mu}_{\vy}\right]\equiv0$.

Note that, \cref{condition3} requires that all the positive samples approach the class means, \textit{i.e.}, $\forall \vx^+ \sim p(\vx^+), f(\vx^+)=\boldsymbol{\mu}_{\vy}$. We then give the following lemma at the state of category-level uniformity.

\begin{lemma} \label{lemma:tight} When it satisfies the category-level uniformity~(\cref{def:cu}) defined on the geometric uniform structure $\mathbf{J}$~(\cref{def:structure}) with dimension $K=L$,  assume $A(J)=0$, for $\forall f \in \mathcal{F}$, the lower bound~(\cref{lemma:lower}) is achieved as

\begin{equation}\label{eq:tight}
    \mathcal{L}_{\mathrm{InfoNCE}}(f) = \mathcal{L}^{\mu}_{\mathrm{CE}}(f)+\log \left(\frac{J}{L}\right)
\end{equation}  
    
\end{lemma}

\begin{proof}
According to category-level uniformity~(\cref{def:cu}), we should have
\begin{equation}\label{eq:conditionhold}
\begin{split}
&\E_{p(\vx)} f(\vx)\equiv \E_{p(\vx^+)} f(\vx^+) \equiv\boldsymbol{\mu}_y, \\
&\E_{p(\vx)}\E_{p(\vx^-)} \exp(f(\vx)^\top f(\vx^-)) \equiv C 
\end{split}
\end{equation}
where the second term is derived from $f(\vx)^\top f(\vx^-) = \mathbf{M}_i^\top\cdot\mathbf{M}_j = C, i \neq j$ in \cref{def:cu}. Note that, the category-level uniformity holds on the joint embedding $p(\vx, \vx^+)$ of contrastive learning in our setup. 

In the proof of \cref{lemma:lower}, (1) holds as we assume M is large enough and $A(J)=0$, (2) holds according to \cref{eq:conditionhold}, (3)(4) holds as $\E_{p(\vx^+)}\left[f(\vx^+)-\boldsymbol{\mu}_{\vy}\right]\equiv0$. As above mentioned, the intra-class variance term $\sqrt{\mathrm{Var}\left(f_\theta(\vx)|\vy\right)}$ is eliminated. We then have the desired results with \cref{eq:tight}.

\end{proof}

\subsection{Proof of \cref{theorem:opt}}

\begin{proof}

On the basis of \cref{lemma:tight}, we can derive our overall loss $\mathcal{L}$ as follows, 

\begin{equation}
\begin{split}
\mathcal{L}(f_\theta,\vx) &=  \mathcal{L}_{\mathrm{InfoNCE}}(f_\theta,\vx,\vx^{+}) + \mathcal{L}_{\mathrm{\method}}(f_\theta,\vx,\hat{\vq})  \\
& = \mathcal{L}^{\mu}_{\mathrm{CE}}(f_\theta, \vx) + \mathcal{L}_{\mathrm{\method}}(f_\theta,\vx,\hat{\vq}) + \log\left(\frac{J}{L}\right)
\end{split}
\end{equation}

Now we focus on analyzing the minimization of the first and the second term as $\log\left(\frac{J}{L}\right)$ is a constant. Here, we assume the temperature $\gamma_{\mathrm{GH}}$ for generating surrogate labels is small enough, so that we can obtain the discrete geometric labels $\hat{\vq}$ in one-hot probabilities.  

For simplicity, we denote the assigned labels as $t$ for all the data points in class $k$, which are consistent as the samples converge to the class means according to \cref{eq:conditionhold}. Let $\hat{\mathcal{L}}(f_\theta,\vx) = \mathcal{L}^{\mu}_{\mathrm{CE}}(f_\theta, \vx_k, \vy) + \mathcal{L}_{\mathrm{\method}}(f_\theta,\vx_k,t)$, we define the optimization problem regarding class $k$ as:

\begin{equation}\label{eq:kce}
\begin{split}
&\min \hat{\mathcal{L}}(f_\theta,\vx_k) = \min \mathcal{L}^{\mu}_{\mathrm{CE}}(f_\theta, \vx_k) + \mathcal{L}_{\mathrm{\method}}(f_\theta,\vx_k,t) \\
&\mathrm{s.t.} \ \ \  \Vert f_\theta(\vx_{k,i}) \Vert^2 = 1, \ \ \ \forall i=1,2,\dots,n_k
\end{split}
\end{equation}

We can then derive 

% \slash\gamma_{\rm CL}
% \slash\gamma_{\rm GC}
\begin{footnotesize}
\begin{equation} \label{eq:gclower}
    \begin{split}
 \hat{\mathcal{L}}(f_\theta,\vx_k) & = \mathcal{L}^{\mu}_{\mathrm{CE}}(f_\theta, \vx_k) + \mathcal{L}_{\mathrm{\method}}(f_\theta,\vx_k,t) \\ 
 &=  - \frac{1}{n_k} \sum_{i=1}^{n_k} \log \frac{\exp \left(f_\theta(\vx_{k,i})^\top \cdot \boldsymbol{\mu}_y \slash\gamma_{\rm CL} \right)}{\sum_{j=1}^{K}\exp\left(f_\theta(\vx_{k,i})^\top\cdot \boldsymbol{\mu}_j\slash\gamma_{\rm CL}\right)} - \frac{1}{n_k} \sum_{i=1}^{n_k} \log \frac{\exp\left(f_\theta(\vx_{k,i})^\top \cdot \mathbf{M}_t\slash\gamma_{\rm \method}\right)}{\sum_{j=1}^{K}\exp\left( f_\theta(\vx_{k,i})^\top\cdot \mathbf{M}_j\slash\gamma_{\rm \method}\right)} \\
&= -  \log \frac{\exp \left(\boldsymbol{\mu}_k^\top \cdot \boldsymbol{\mu}_k\slash\gamma_{\rm CL}\right)}{\sum_{j=1}^{K}\exp \left( \boldsymbol{\mu}_k^\top \cdot \boldsymbol{\mu}_j\slash\gamma_{\rm CL}\right)} - \log \frac{\exp \left(\boldsymbol{\mu}_k^\top \cdot \mathbf{M}_t\slash\gamma_{\rm \method}\right)}{\sum_{j=1}^{K}\exp \left( \boldsymbol{\mu}_k^\top \cdot \mathbf{M}_j\slash\gamma_{\rm \method}\right)}
    \end{split}
\end{equation}
\end{footnotesize}

According to \cref{eq:conditionhold}, the constraints of \cref{eq:kce} are equivalent with $\Vert \boldsymbol{\mu}_k \Vert^2 = 1$. We can have the Lagrange function as:

\begin{equation}
    \tilde{\mathcal{L}} = - \log \frac{\exp \left(\boldsymbol{\mu}_k^\top \cdot \boldsymbol{\mu}_k\slash\gamma_{\rm CL}\right)}{\sum_{j=1}^{K}\exp \left( \boldsymbol{\mu}_k^\top \cdot \boldsymbol{\mu}_j\slash\gamma_{\rm CL}\right)} -  \log \frac{\exp \left(\boldsymbol{\mu}_k^\top \cdot \mathbf{M}_t\slash\gamma_{\rm \method}\right)}{\sum_{j=1}^{K}\exp \left( \boldsymbol{\mu}_k^\top \cdot \mathbf{M}_j\slash\gamma_{\rm \method}\right)} + \eta_k (\Vert \boldsymbol{\mu}_{k} \Vert^2-1)
\end{equation}
where $\eta_k$ is the Lagrange multiplier.

% \sum_{i=1}^{n_k} \eta_k (\Vert \boldsymbol{\mu}_{k} \Vert^2-1)

We consider its gradient with respect to $\boldsymbol{\mu}_k$ as:

\begin{footnotesize}
\begin{equation}
\begin{split}
    \frac{\partial \tilde{\mathcal{L}}}{\partial \boldsymbol{\mu}_{k}} &= \frac{1}{\gamma_{\rm CL}}\left[- (1 - m_k) \cdot \boldsymbol{\mu}_k + \sum_{i \neq k}^K m_i \cdot \boldsymbol{\mu}_i\right]+\frac{1}{\gamma_{\rm \method}}\left[- (1 - n_k)\cdot \mathbf{M}_t + \sum_{i \neq t}^K n_i \cdot \mathbf{M}_i\right] + (\frac{1}{\gamma_{\rm CL}}+2\eta_k)\boldsymbol{\mu}_k  \\ 
&= \frac{1}{\gamma_{\rm CL}} \sum_{i \neq k}^K m_i(\boldsymbol{\mu}_i-\boldsymbol{\mu}_k) + \frac{1}{\gamma_{\rm \method}}\sum_{i \neq t}^K n_i(\mathbf{M}_i-\mathbf{M}_t) + (\frac{1}{\gamma_{\rm CL}}+2\eta_k)\boldsymbol{\mu}_k \\
\end{split}
\end{equation}\end{footnotesize}
where $ m_i = \frac{\exp\left(\boldsymbol{\mu}_k^\top \cdot \boldsymbol{\mu}_i\slash\gamma_{\rm CL}\right)}{\sum_{j=1}^K\exp\left(\boldsymbol{\mu}_k^\top \cdot \boldsymbol{\mu}_j\slash\gamma_{\rm CL}\right)}, \ \ \  n_i = \frac{\exp\left(\boldsymbol{\mu}_k^\top \cdot \mathbf{M}_i\slash\gamma_{\rm \method}\right)}{\sum_{j=1}^K\exp\left(\boldsymbol{\mu}_k^\top \cdot \mathbf{M}_j\slash\gamma_{\rm \method}\right)}$.

When it satisfies the category-level uniformity~(\cref{def:cu}) defined on the geometric uniform classifier $\mathbf{M}$~(\cref{def:structure}), we can obtain $\boldsymbol{\mu}_k = \mathbf{M}_k$. 

Multiplying $\mathbf{M}_j$ over the gradients~($j\neq k,j \neq t$):
\begin{footnotesize}
\begin{equation}
    \begin{split}
\frac{\partial \tilde{\mathcal{L}}}{\partial \boldsymbol{\mu}_{k}} \cdot \mathbf{M}_j 
&= \frac{1}{\gamma_{\rm CL}}\sum_{i \neq k}m_i(\boldsymbol{\mu}_i \cdot \mathbf{M}_j-\boldsymbol{\mu}_k\cdot \mathbf{M}_j) + \frac{1}{\gamma_{\rm \method}}\sum_{i \neq t}n_i(\mathbf{M}_i\cdot \mathbf{M}_j-\mathbf{M}_t\cdot \mathbf{M}_j) + (\frac{1}{\gamma_{\rm CL}}+2\eta_k)\boldsymbol{\mu}_k\cdot \mathbf{M}_j \\
&= \frac{1}{\gamma_{\rm CL}}\sum_{i \neq k}m_i(\mathbf{M}_i \cdot \mathbf{M}_j-\mathbf{M}_k\cdot \mathbf{M}_j) + \frac{1}{\gamma_{\rm \method}}\sum_{i \neq t}n_i(\mathbf{M}_i\cdot \mathbf{M}_j-\mathbf{M}_t\cdot \mathbf{M}_j) + (\frac{1}{\gamma_{\rm CL}}+2\eta_k)\mathbf{M}_k\cdot \mathbf{M}_j \\
&= (m_j + n_j)(1-C) + (\frac{1}{\gamma_{\rm CL}}+2\eta_k)C
    \end{split}
\end{equation}\end{footnotesize}
where $C$ is defined in \cref{def:structure}. We can have the probabilities $m_j$, $n_j$ as

\begin{equation}
    m_j = n_j = \frac{1}{1+(K-1)\exp\left(C-1\right)}, \ j\neq k
\end{equation}

Let $\eta_k = \frac{C-1}{C+(L-1)C\exp\left(C-1\right)} - \frac{1}{2\gamma_{\rm CL}}$, we can have $\frac{\partial \tilde{\mathcal{L}}}{\partial \boldsymbol{\mu}_{k}} \cdot \mathbf{M}_j=0$. With $\mathbf{M}_j \neq 0$, we should have $\frac{\partial \tilde{\mathcal{L}}}{\partial \boldsymbol{\mu}_{k}}=0$. Similarly applying to other classes, we can have $\frac{\partial \tilde{\mathcal{L}}}{\partial \boldsymbol{\vx}}=0$.

Eventually, we can obtain the minimizer $\hat{\mathcal{L}}^*(f_\theta,\vx)$ as:

\begin{equation}
    \begin{split}
        \mathcal{L}^*(f_\theta,\vx)  = -\sum_{k=1}^{K}2\boldsymbol{\pi}_{l}^{\vy}\log\left(\frac{1}{1+(K-1)\exp(C-1)}\right) + \log\left(\frac{J}{L}\right)
    \end{split}
\end{equation}

\end{proof}

\subsection{Proof of \cref{lemma:mc}}

\begin{proof}
Assume the samples follow the uniform distribution $n_1=n_2=\dots=n_{L_H}=n_H$, $n_{L_H+1}=n_{L_H+2}=\dots=n_{L}=n_T$ in head and tail classes respectively. Assume the imbalance ratio $\frac{n_H}{n_T} \rightarrow +\infty$ and the dimenson satisfies $K\geq L$. As proof in \cite{fang2021exploring}, we can have 

\begin{equation}
        \lim \boldsymbol{\mu}_i - \boldsymbol{\mu}_{j} = \mathbf{0}_L, \ \forall L_H \leq i \leq j \leq L, \nonumber
\end{equation}
when the cross-entropy loss achieves the minimizer. Then we can have the lower bound~(\cref{lemma:lower}) of $\mathcal{L}_{\mathrm{InfoNCE}}$ achieves minimum when the above equation holds, \textit{i.e.}, minority class means collapse to an identical vector.

\end{proof}

\subsection{Discussions of \cref{lemma:mc}}

Intrinsically, \cref{lemma:mc} is an extreme analysis to characterize the trend under the increasing imbalanced ratios between the majority classes and the minority classes. The staged-wise imbalancing condition is to reach the final compact form about the minority collapse, and more practical long-tailed distribution only reaches the intermediate deduction with much understanding effort, which is even not solved in the current theoretical analysis in supervised long-tailed learning~\citep{fang2021exploring}. The $\frac{N_H}{N_t}\rightarrow + \infty$ binds with the $\lim$ in the equation is for extreme analysis, but is not for the practical requirement.

\subsection{Applicability of \cref{theorem:opt}}

\revise{Our theorem and analyses are specific to contrastive learning. In terms of other non-contrastive SSL methods, we empirically show the superiority of our method on long-tailed data distribution in \cref{tab:moressl}. Although it might not be straightforward to extend the theory to non-contrastive SSL methods, an explanation about the consistent superiority is that some non-contrastive methods still exhibit similar representation disparity with their contrastive counterpart, and our proposed method can similarly reallocate the geometric distribution to counteract the distorted embedding space. Specially, the recent study~\citep{garrido2023duality} theoretically and emprically explore the equivalence between contrastive and non-contrastive criterion, which may shed light on the intrinsic mechanism of how our \methodspace benefits non-contrastive paradigm.}

% \newpage

\section{Algorithms} \label{appendix:algorithm}

\subsection{Algorithm of Surrogate Label Allocation}

We summarize surrogate label allocation in \cref{algorithm:allocation}.

\setlength{\textfloatsep}{10pt}
\begin{algorithm}[!htb] 
  \caption{Surrogate Label Allocation.} \label{algorithm:allocation}
    {\bf Input:} geometric cost matrix 
    % $\exp(\lambda\log\mathbf{Q}) \in \mathbb{R}^{K \times N}_{+}$ 
    $\exp(\lambda\log\mathbf{Q})$
    with $\mathbf{Q}=[\vq_1,\dots,\vq_N]$, marginal distribution constraint $\boldsymbol{\pi}$, Sinkhorn regularization coefficient $\lambda$, Sinkhorn iteration step $E_s$  \\
    {\bf Output:} Surrogate label matrix $\hat{\mathbf{Q}}$
\begin{algorithmic}[1]
    \STATE Set scaling vectors $\vu \gets \frac{1}{K}\cdot\mathbbm{1}_K, \vv \gets \frac{1}{N}\cdot\mathbbm{1}_N$.
    \STATE Set distribution constraints $\vr \gets \frac{1}{N} \cdot \mathbbm{1}_N, \vc \gets \boldsymbol{\pi}$.
    
    \FOR{iteration $i=0,1,\ldots,E_s$} 
    \STATE $\vu \gets \log \vc -\log \left(\left(\exp(\lambda\log\mathbf{Q})\right) \cdot \exp(\vv)\right)$.
    \STATE $\vv \gets \log \vr -\log \left(\left(\exp(\lambda\log\mathbf{Q})\right)^{\top} \cdot \exp(\vu)\right)$.
    \ENDFOR
    % \STATE Compute $\hat{\mathbf{Q}} = N \cdot \mathrm{diag}(\vu)\exp(\lambda\log\mathbf{Q})\mathrm{diag}({\vv})$.
    % \RETURN $\hat{\mathbf{Q}}$  
    \STATE \textbf{return} $\hat{\mathbf{Q}}= N \cdot \mathrm{diag}(\vu)\exp(\lambda\log\mathbf{Q})\mathrm{diag}({\vv})$ 

\end{algorithmic}
\end{algorithm}

\subsection{Algorithm of Geometric Harmonization}

We summarize the complete procedure of our \methodspace method in \cref{algorithm:method}.

\begin{algorithm}[!htb] 
  \caption{Our proposed \method.} \label{algorithm:method}
    {\bf Input:} dataset $\mathcal{D}$, number of epochs $E$, number of warm-up epochs $E_w$, geometric uniform classifier $\mathbf{M}$, a self-supervised learning method $\mathcal{A}$  \\
    {\bf Output:} pretrained model parameter $\theta_{E}$\\
    {\bf Initialize:} model parameter $\theta_{0}$
\begin{algorithmic}[1]
    \STATE Warm up model $\theta$ for $E_w$ epochs according to $\mathcal{A}$.
    
    \FOR{epoch $e=E_w,E_w+1,\ldots,E$}
    \STATE Compute the geometric predictions $\mathbf{Q}$ for input samples.
    \STATE Compute the surrogate class prior $\boldsymbol{\pi}$ on training dataset $\mathcal{D}$.
    \FOR{mini-batch $k=1,2,\ldots,B$}
    \STATE Obtain the surrogate label $\hat{\mathbf{Q}}$ by \cref{algorithm:allocation}.
    \STATE Compute $\mathcal{L}_{\mathrm{CL}}$  according to $\mathcal{A}$ and the proposed $\mathcal{L}_{\mathrm{\method}}$ according to \cref{eq:overall}.
    \STATE Uptate model $\theta$ by minimizing $\mathcal{L}_{\mathrm{CL}}+\mathcal{L}_{\mathrm{\method}}$.
    \ENDFOR
    \ENDFOR

\end{algorithmic}
\end{algorithm}

\section{Supplementary Experimental Setups}\label{apd:exp}

\subsection{Dataset Statistics}\label{appendix:dataset}

% \notes{Zhihan}{add disjoint group partition details}

We conduct experiments on three benchmark datasets for long-tailed learning, including CIFAR-100-LT~\citep{cao2019learning}, ImageNet-LT~\citep{liu2019large} and Places-LT~\citep{liu2019large}. For small-scale datasets, we adopt the widely-used CIFAR-100-LT with the imbalanced factor of 100, 50 and 10~\citep{cao2019learning}. 

In \cref{table:data_stat}, we summarize the benchmark datasets used in this paper. Long-tailed versions of CIFAR-100~\citep{krizhevsky2009learning,fan2022fedskip} are constructed following the exponential distribution. For large-scale datasets, ImageNet-LT~\citep{liu2019large} has 115.8K images with 1000 categories, ranging from 1,280 to 5 in terms of class cardinality and Places-LT~\citep{liu2019large} contains 62,500 images with 365 categories, with the sample number per category ranging from 4,980 to 5. The large-scale datasets follow Pareto distribution.

As for fine-grained group partitions, we divide each dataset to Many/Medium/Few according to the class cardinality. Concretely, we choose that the largest 34 classes for Many group, the medium 33 classes for Medium group and the smallest 33 classes for Few group on CIFAR-100-LT. On ImageNet-LT and Places-LT, we define Many group with class number over 100, Medium group with 20-100 samples, Few group as under 20 samples~\cite{liu2019large}.

\begin{table}[!htb]
\centering
\caption{Statistics of the benchmark long-tailed datasets. Exp represents exponential distribution.} \label{table:data_stat}
\begin{tabular}{lccccc}
\toprule[1.5pt]
Dataset           & \# Class & Type & Imbalanced Ratio & \# Train data  & \# Test data  \\ \midrule
CIFAR-100-LT-R100 & 100   & Exp            & 100              & 10847  & 10000 \\
CIFAR-100-LT-R50  & 100   & Exp            & 50               & 12608  & 10000 \\
CIFAR-100-LT-R10  & 100   & Exp            & 10               & 19573  & 10000  \\ \midrule
ImageNet-LT       & 1000  & Pareto         & 256              & 115846 & 50000 \\ \midrule
Places-LT         & 365   & Pareto         & 996              & 62500  & 36500 \\ \bottomrule[1.5pt] 
\end{tabular}
\end{table}

\subsection{Linear probing statistics on the large-scale dataset}\label{appendix:linear}

 The 100-shot evaluation follows the setting in previous works \citep{jiang2021self,zhou2022contrastive}. As shown in \cref{tab:100stat}, full-shot evaluation requires 10x - 30x the amount of data compared with the pre-training dataset, which might not be very practical. In contrast, the scale of 100-shot data is consistent with the pre-training dataset. We also present full-shot evaluation in \cref{exp:fullshot}.

\begin{table}[!htb]
\centering

\caption{Statistics of linear probing on the large-scale dataset.} \label{tab:100stat}
\begin{tabular}{lccccc}
\toprule[1.5pt] 
Dataset           & \# Class  & \# Training data & \# 100-shot data & \# full-shot data  & \# Test data  \\ \midrule
ImageNet-LT       & 1000   & 115,846 & 100,000 & 1,261,167   & 50,000 \\ 
Places-LT         & 365    & 62,500  & 36,500  & 1,803,460  & 36,500 \\ \bottomrule[1.5pt]
\end{tabular}
\end{table}

\subsection{Implementation Details}\label{appendix:impdetail}

\textbf{Toy Experiments.} We use a 2-Layer ReLU network with 20 hidden units and 2 output units for visualization. For \cref{fig:toy}, the SimCLR algorithm~\citep{chen2020simple} is adopted in the warm-up stage with proper Gaussian noise as augmentation. After the warm-up stage, we train \methodspace according to \cref{eq:overall}. We use the orthogonal classifier [(1,1),(-1,1),(-1,-1),(1,-1)] as the geometric uniform structure. For \cref{fig:toyR}, only the SimCLR algorithm is adopted for representation learning.

\textbf{More Experimental Setup for Main Results.} (SimCLR, Focal, SDCLR, DnC, BCL) In our experiments, we defaultly set the contrastive learning temperature $\gamma_{\rm CL}$ as 0.2 and the smoothing coefficient $\beta$ as 0.999 for training stability. For updating the marginal distribution constraint $\boldsymbol{\pi}$, we compute every 20 epochs on CIFAR-100-LT due to the small data size. On ImageNet-LT and Places-LT, we compute $\boldsymbol{\pi}$ every training epoch. Following previous work~\cite{jiang2021self,zhou2022contrastive}, we adopt a 2-layer MLP as the projector with 128 output dimension. For default data augmentations of contrastive learning,  random crop ranging from [0.1, 1], random horizontal flip, color jitter with probability as 0.8 and strength as 0.4 are adopted on CIFAR-100-LT. Random crop ranging from [0.08, 1], random horizontal flip, color jitter with probability as 0.8 and strength as 0.4 and the gaussian blur with probability as 0.5 are adopted on ImageNet-LT and Places-LT.
% \notes{TODO}{add details}

\textbf{Linear Probing Evaluation.} We follow~\citet{zhou2022contrastive} to conduct Adam optimizer for 500 epochs based on batch size 128, weight decay factor $5 \times 10^{-6}$ and the learning rate decaying from $10^{-2}$ to $10^{-6}$. For few-shot evaluation on ImageNet-LT and Places-LT, we use the same subsampled 100-shot subsets proposed in~\citep{zhou2022contrastive}.

\subsection{Focal Loss}

Focal loss~\citep{lin2017focal} is discussed and compared in \citep{jiang2021self,zhou2022contrastive} in the context of self-supervised long-tailed learning. Specifically, we use the term inside log(·) of SimCLR loss as the likelihood to replace the probabilistic term of the supervised Focal loss and obtain the self-supervised Focal loss as:

 \begin{equation}
\mathcal{L}_{focal} =-\frac{1}{|\mathcal{D}|} \sum_{\vx \in \mathcal{D}}(1-\vp)^{\gamma_{\rm F}} \log(\vp), \ \ \ 
\vp= \frac{\exp \left(f(\vx)^\top f(\vx^{+})/\gamma_{\rm F}\right)}{\sum_{x^{-} \in \mathcal{X}_b^{-}\cup \{\vx^{+}\}} \exp \left(f(\vx)^\top f(\vx^{-})/\gamma_{\rm F}\right)} \nonumber
\end{equation}
where $\gamma_{\rm F}$ is a temperature factor and $\mathcal{X}_b$ denotes the negative sample set. We defaultly set $\gamma_{\rm F}$ as 2 in all experiments.

\subsection{Toy Experiments on Various Imbalanced Ratios}

In Figure~\ref{fig:toyR}, we provide a concrete visualization on a 2-D toy dataset that the sample-level uniformity of the contrastive learning loss leads to the more space invasion of head classes and space collapse of tail classes with increasing the imbalance ratios. According to the results, we can observe that the head classes gradually occupy the embedding space as the imbalanced ratios increase. This further demonstrates the importance of designing robust self-supervised learning method to counteract the distorted embedding space in the long-tailed context.

\begin{figure}[!htb]
   
    \centering
    \subfigure[R=1(Balanced)]{
    \includegraphics[height=0.22\textwidth,width=0.22\textwidth]{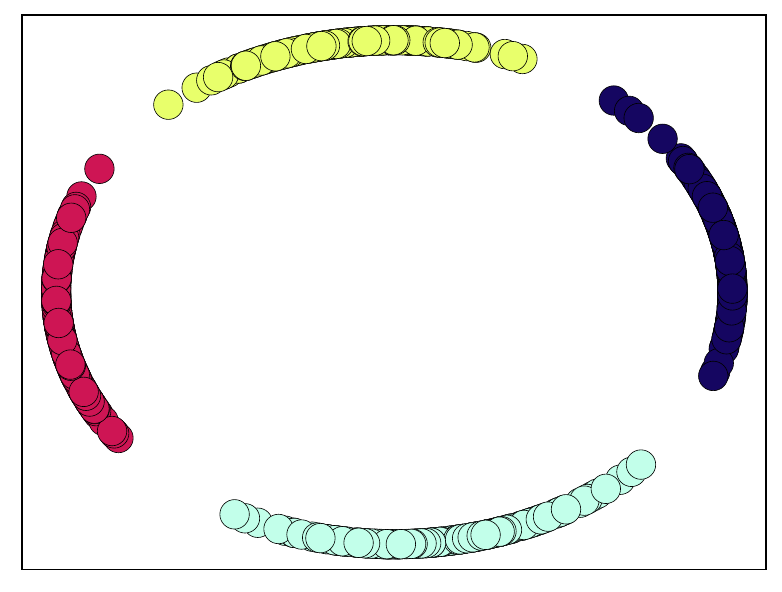}
    }
    \hspace{10mm}
    % \hspace{-10mm}
    \subfigure[R=4]{
    \includegraphics[height=0.22\textwidth,width=0.22\textwidth]{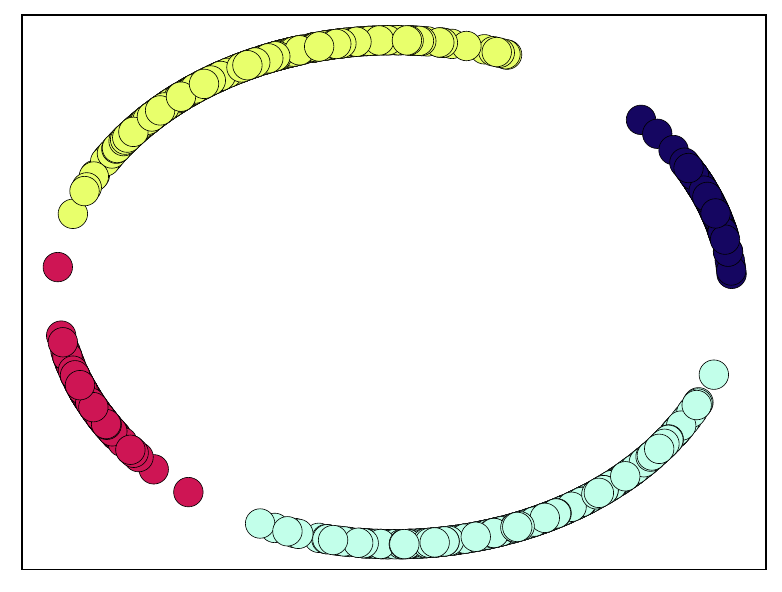}
    }
    % \vspace*{-8pt}

    \subfigure[R=16]{
    \includegraphics[height=0.22\textwidth,width=0.22\textwidth]{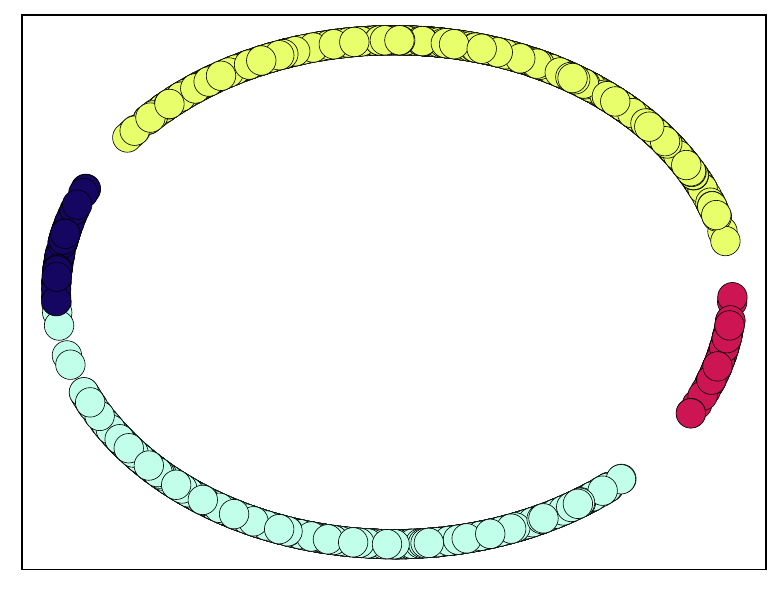}
    }
    \hspace{10mm}
    \subfigure[R=64]{
    \includegraphics[height=0.22\textwidth,width=0.22\textwidth]{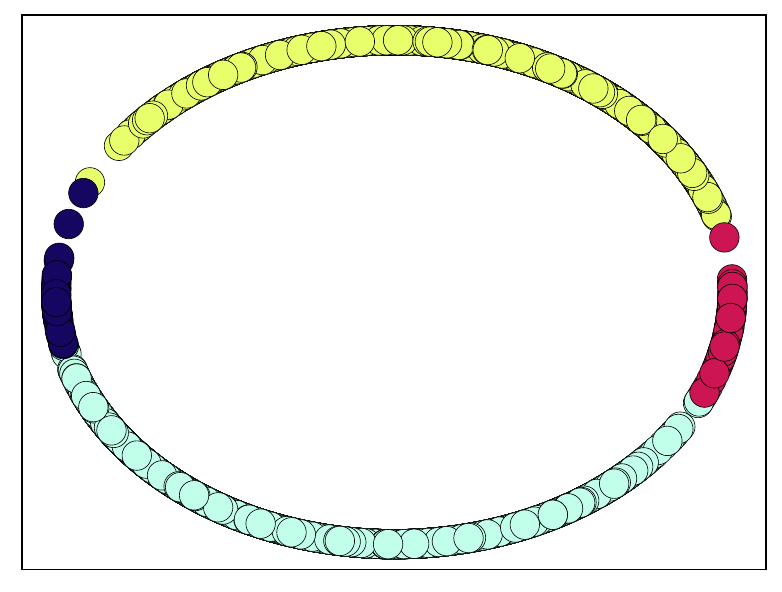}
    }
    % \vspace*{-8pt}
    \caption{Visualization of the embedding space learnt by vanilla contrastive learning loss on the 2-D imbalanced synthetic dataset with different imbalanced ratios~(1,4,16,64). As the ratio increases, head classes gradually occupy the embedding space with the collapse of the tail classes. }\label{fig:toyR}
    % \vspace*{-14pt}
\end{figure}

% \newpage
\section{Additional Experimental Results and Further Analysis}
\label{appendix:results}

\subsection{Error Bars for the Main Results}

In this part, we present main results with error bars calculated over 5 trials.

\begin{table}[!htb]

\caption{Linear probing results (average accuracy, \%) over 5 trials on CIFAR-LT with different imbalanced ratios~(100,50,10), ImageNet-LT and Places-LT.}\label{tab:errorbar}
\begin{tabular}{c|ccc|c|c}
\toprule[1.5pt]
          & CIFAR-LT-R100 & CIFAR-LT-R50 & CIFAR-LT-R10 & ImageNet-LT & Places-LT  \\\midrule
SimCLR    & 50.72$\pm$0.26    & 52.24$\pm$0.31   & 55.67$\pm$0.44   & 36.65$\pm$0.16  & 33.61$\pm$0.12 \\
+\method & 53.96$\pm$0.23    & 55.42$\pm$0.22   & 57.36$\pm$0.39   & 38.28$\pm$0.13  & 34.33$\pm$0.10 \\\midrule
Focal     & 51.04$\pm$0.27    & 52.22$\pm$0.38   & 56.23$\pm$0.45   & 37.49$\pm$0.11  & 33.65$\pm$0.14 \\
+\method  & 53.92$\pm$0.19    & 55.06$\pm$0.28   & 58.05$\pm$0.28   & 38.92$\pm$0.14  & 34.42$\pm$0.17 \\\midrule
SDCLR     & 52.87$\pm$0.22    & 53.87$\pm$0.21   & 55.44$\pm$0.25   & 36.25$\pm$0.18  & 33.99$\pm$0.14 \\
+\method  & 54.81$\pm$0.26    & 55.34$\pm$0.28   & 56.97$\pm$0.34   & 38.53$\pm$0.14  & 34.70$\pm$0.10 \\\midrule
DnC       & 52.52$\pm$0.32    & 53.21$\pm$0.35   & 57.59$\pm$0.36   & 37.23$\pm$0.21  & 33.90$\pm$0.18 \\
+\method    & 54.88$\pm$0.23    & 56.33$\pm$0.31   & 58.94$\pm$0.25   & 38.67$\pm$0.19  & 34.52$\pm$0.23 \\\midrule
BCL       & 56.45$\pm$0.40    & 57.18$\pm$0.26   & 59.12$\pm$0.28   & 38.33$\pm$0.10  & 34.76$\pm$0.15 \\
+\method    & 57.65$\pm$0.33    & 59.00$\pm$0.33   & 60.34$\pm$0.29   & 39.95$\pm$0.15  & 35.32$\pm$0.17 \\ \bottomrule[1.5pt]
\end{tabular}
\end{table}

\subsection{Convergence of the Surrogate Label Allocation}

In \cref{tab:convergence}, we provide the experiments to verify the convergence of the Sinkhorn-Knopp algorithm, which adopts the criterion as the stopping reference.

\begin{table}[!htb]
\centering

\caption{The value of $e$ during the convergence of surrogate label allocation on CIFAR-LT-R100.} \label{tab:convergence}
\begin{tabular}{c|cccccccc}
\toprule[1.5pt]
Iter & 0 & 10 & 20 & 30 & 50 & 70 & 100 & 150 \\
\midrule
$e$ & 67.89 & 4.28 & 0.53 & 0.076 & 0.0054 & 0.0005 & 2.08$\times 10^{-5}$ & 3.58$\times 10^{-7}$ \\
\bottomrule[1.5pt]
\end{tabular}
\end{table}

We define the criterion $e = sum(| u./u' - 1 |)$ as the relative changes of one scaling vectors $u$, where $u'$ represents the vector in the latest iteration. Then, the algorithm converges as the criterion $e \rightarrow 0$. As shown in \cref{tab:convergence}, we can see that the criterion diminishes rapidly. Let $e < 10^{-6}$ represent the indicator of the convergence, we further obtain the averaging convergence iterations as $141\pm45$ (statistics under 1000 runs). In practice, we set the default Sinkhorn iterations as 300 to guarantee the convergence, as detailed in \cref{sec:exp}.

\subsection{Empirical Comparison with More Baselines}\label{appendix:moresslbaseline}

\revise{In \cref{tab:moresslbaseline}, we conduct a range of experiments to compare PMSN\citep{kukleva2023temperature} and TS~\citep{assran2023hidden} with our proposed GH on CIFAR-LT with different imbalanced ratios. }

\begin{table}[!htb]
\centering
\revise{\caption{Linear probing accuracy of more SSL-LT baselines on CIFAR-100-LT with different imbalanced ratios.}\label{tab:moresslbaseline}
\vspace{.4em}
\begin{tabular}{c|c|c|c|c|c}
\toprule[1.5pt]
 & Method & Many & Med & Few & Avg \\
\midrule[0.6pt]\midrule[0.6pt]
\multirow{5}{*}{\rotatebox[origin=c]{90}{CIFAR-R100}}
& SimCLR & 54.97	&49.39	&47.67	&50.72 \\
& SimCLR+TS &55.53	&50.33	&50.06	&52.01\\
& PMSN &55.62	&52.12	&49.85	&52.56 \\
& \cellcolor{greyL}SimCLR\textbf{+GH} & \cellcolor{greyL}57.38	&\cellcolor{greyL}52.27	&\cellcolor{greyL}\textbf{52.12}	&\cellcolor{greyL}53.96 \\
& \cellcolor{greyL}SimCLR+TS\textbf{+GH} &\cellcolor{greyL}\textbf{57.44}	&\cellcolor{greyL}\textbf{52.76}	&\cellcolor{greyL}51.79	&\cellcolor{greyL}\textbf{54.03} \\
\midrule[0.6pt]
\multirow{5}{*}{\rotatebox[origin=c]{90}{CIFAR-R50}}
& SimCLR & 56.00	&50.48	&50.12	&52.24 \\
& SimCLR+TS &56.44	&52.58	&51.91	&53.67\\
& PMSN &56.76	&52.52	&53.09	&54.15 \\
& \cellcolor{greyL}SimCLR\textbf{+GH} &\cellcolor{greyL}\textbf{58.88}	&\cellcolor{greyL}53.00	&\cellcolor{greyL}54.27	&\cellcolor{greyL}55.42\\
& \cellcolor{greyL}SimCLR+TS\textbf{+GH} &\cellcolor{greyL}58.47	&\cellcolor{greyL}\textbf{54.61}	&\cellcolor{greyL}\textbf{54.70}	&\cellcolor{greyL}\textbf{55.95}\\
\midrule[0.6pt]
\multirow{5}{*}{\rotatebox[origin=c]{90}{CIFAR-R10}}
& SimCLR & 57.85	&55.06	&54.03	&55.67\\
& SimCLR+TS &58.26	&56.24	&54.97	&56.51\\
& PMSN &56.91	&54.61	&55.67	&55.74 \\
& \cellcolor{greyL}SimCLR\textbf{+GH} & \cellcolor{greyL}59.26	&\cellcolor{greyL}56.91	&\cellcolor{greyL}55.85	&\cellcolor{greyL}57.36 \\
& \cellcolor{greyL}SimCLR+TS\textbf{+GH} &\cellcolor{greyL}\textbf{59.44}	&\cellcolor{greyL}\textbf{57.15}	&\cellcolor{greyL}\textbf{56.48}	&\cellcolor{greyL}\textbf{57.71} \\
\bottomrule[1.5pt]
\end{tabular}}
\end{table}

\revise{From the results, we can see that the proposed method consistently outperforms PMSN\citep{kukleva2023temperature} and TS~\citep{assran2023hidden} across different imbalanced ratios on CIFAR-LT. Besides, we can observe that combining GH and TS~\citep{assran2023hidden} consistently improves the performance of contrastive learning on CIFAR-LT.}

\subsection{Empirical Comparison with K-Means Algorithm}

\revise{K-means algorithm~\citep{hartigan1979algorithm} tends to generate clusters with relatively uniform sizes, which will affect the cluster performance under the class-imbalanced scenarios~\citep{liang2012k}. To gain more insights, we conduct empirical comparisons using K-means as the clustering algorithm and evaluate the NMI score with ground-truth labels and the linear probing accuracy on CIFAR-LT-R100. }

\begin{table}[!htb]
\centering
\revise{
\caption{Linear probing accuracy and NMI score on CIFAR-100-LT-R100.}
\label{tab:kmeans}
\begin{tabular}{lcc}
\toprule[1.5pt]
Method        & Accuracy                      & NMI score                                             \\ \midrule
SimCLR        & 50.72                     & 0.28                                        \\
+K-means  & 51.44 & 0.35  \\
+\method      & 53.96                     & 0.50                                          \\ \bottomrule[1.5pt]
\end{tabular}}
\end{table}

\revise{From the results, we can see that K-means generates undesired assignments with lower NMI score and achieves unsatisfying performance compared with our GH. This observation is consistent with previous studies~\citep{liang2012k}.}

\subsection{Compatibility on the Class-Balanced Data}

\revise{
In \cref{tab:balanced}, we present the results on the balanced dataset CIFAR-100 across different methods.} 

\begin{table}[!htb]
\centering
\caption{Linear probing on class-balanced CIFAR-100. We report Accuracy($\%$) for comparison.}\label{tab:balanced}\vspace*{-5pt}
\resizebox{\textwidth}{!}{
\begin{tabular}{c|cc|cc|cc|cc|cc}
\toprule[1.5pt]
Method & SimCLR        & +\method           & Focal         & +\method            & SDCLR         & +\method            & DnC         & +\method              & BCL           & +\method      \\\midrule[0.6pt]
Accuracy    & 66.75  & 66.41  &  66.42  & 66.79  & 65.46  & 66.17   &  67.78  & 67.57  & 69.16  & 69.33  \\\bottomrule[1.5pt]
\end{tabular}}
% \vspace*{-20pt}
% \vspace*{-10pt}
\end{table}

\revise{From the results, we can see that \methodspace shows comparable performance with the baseline methods when the data distribution is balanced. According to the neural collapse theory~\cite{papyan2020prevalence}, well-trained neural networks can inherently produce the category-level uniformity on class-balanced data. As expected, our \methodspace will degenerate to the vanilla SSL baselines as the geometric labels can easily be aligned with the latent ground-truth labels. The empirical findings are also consistent with recent explorations~\cite{kasarla2022maximum} in supervised learning context. \revise{Besides, the minor decrease in performance could potentially be attributed to some random factors during training or the negligible effect of GH loss as it might not reach an absolute zero value.}}

\subsection{Computational Cost}

In \cref{tab:ccost}, we present the mini-batch training time of different baseline methods on CIFAR-100-LT, ImageNet-LT and Places-LT. 

\begin{table}[!htb]
\centering
\caption{The time cost~(seconds) of mini-batch training on CIFAR-100-LT, ImageNet-LT and Places-LT.}\label{tab:ccost}
\resizebox{\textwidth}{!}{%
\begin{tabular}{c|cc|cc|cc|cc|cc}
\toprule[1.5pt]
Dataset      & SimCLR &+\method  & Focal &+\method & SDCLR &+\method & DnC &+\method & BCL &+\method \\\midrule[0.6pt]\midrule[0.6pt]
CIFAR-LT    & 0.38 &0.41                    &    0.37 &0.40       & 0.42 &0.47         & 0.39  &  0.41     & 0.38 &0.41                \\
ImageNet-LT& 0.76 &0.79                   &    0.75 &0.77            & 0.94 &1.01         &  0.76 &   0.79     & 0.76 &0.78                \\
Places-LT   & 0.72 &0.75                   & 0.76 &0.78              &  1.00 &1.05        &  0.72 &   0.76    & 0.72 &0.75        \\\bottomrule[1.5pt]    
\end{tabular}
}
\end{table}

In our runs, the proposed \methodspace only incurs a minor computational overhead on CIFAR-100-LT, ImageNet-LT and Places-LT, respectively, which is relatively lightweight compared to the total computational cost of the contrastive baselines. This indicates the great potential of \methodspace to collaborate with more SSL methods to acquire the robustness on data imbalance in a low-cost manner. 

\subsection{Ablations on Hyper-parameters}

In this part, we present ablation studies \textit{w.r.t.} temperature $\gamma_{\mathrm{GH}}$, regularization coefficient $\lambda$ and Sinkhorn iteration $E_s$ on CIFAR-LT.

% \vspace{-15pt}
\begin{figure}[!htb]
    \centering
    % \hspace{0.25cm}
    \subfigure{
    \begin{minipage}{0.28\textwidth}
    \centering
    \includegraphics[width=\linewidth]{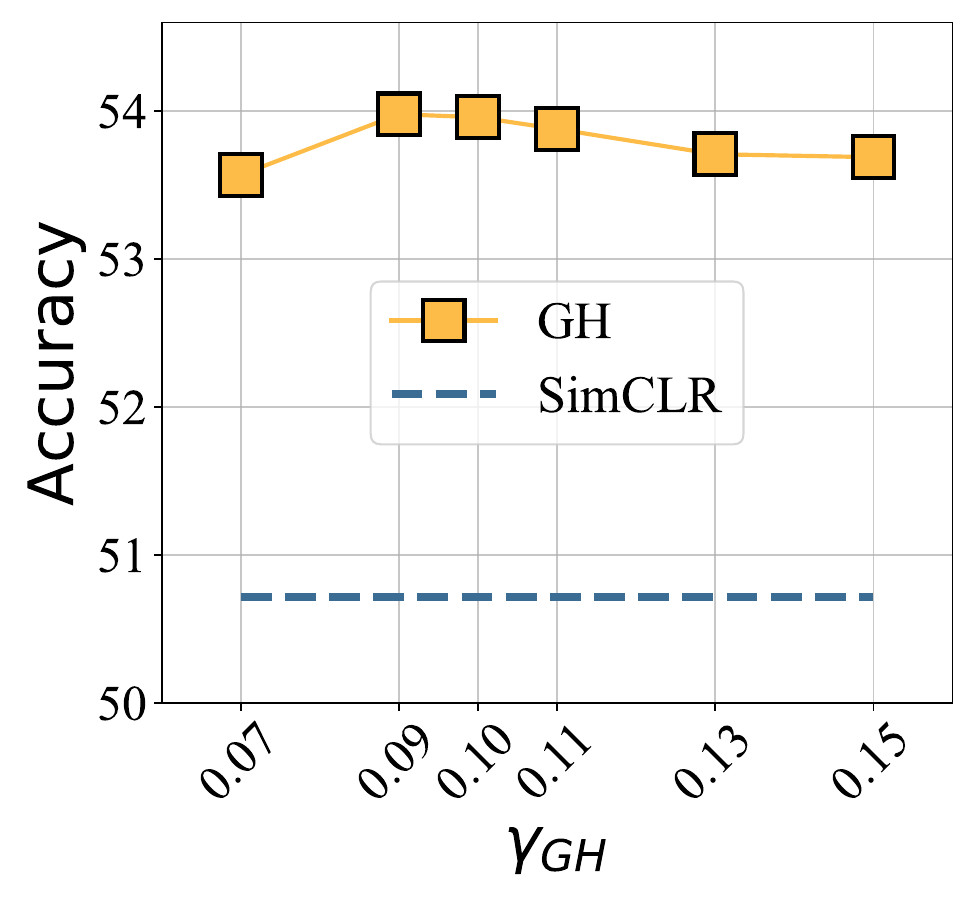}
    \end{minipage}
    }
    % \hfill
    \hspace{0.3cm}
    \subfigure{
    \begin{minipage}{0.28\textwidth}
    \centering
    \includegraphics[width=\linewidth]{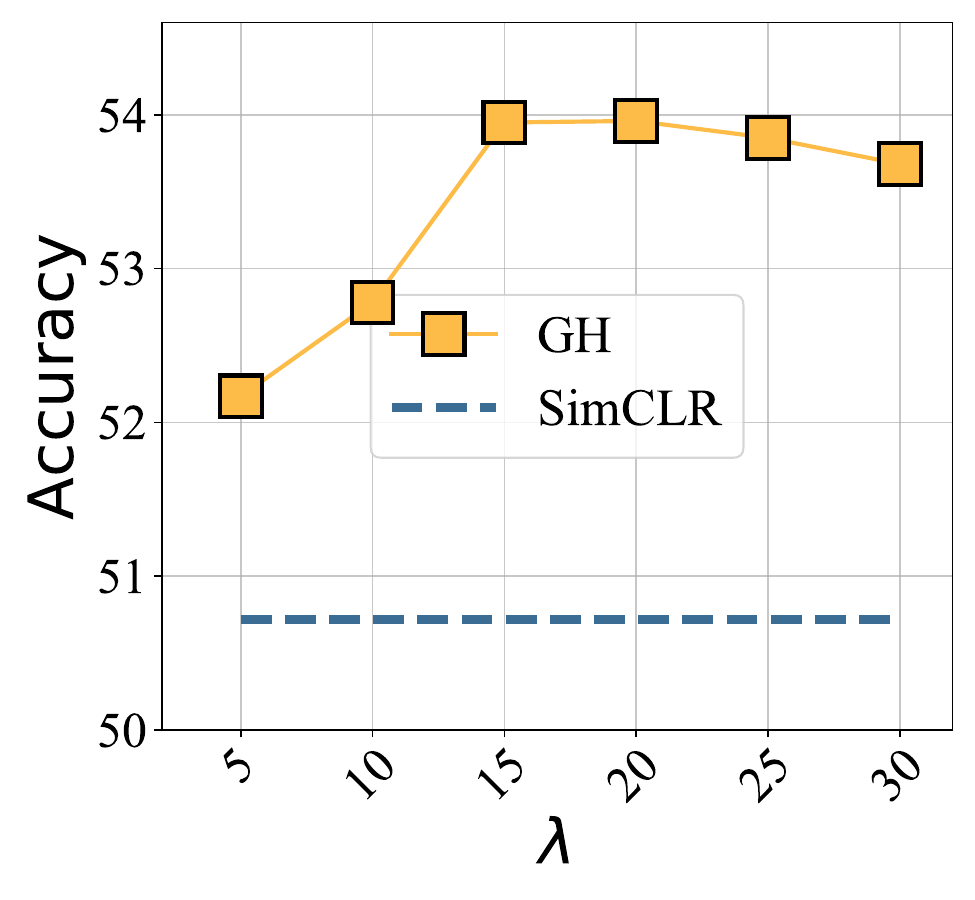}
    \end{minipage}
    }
    % \hfill
    \hspace{0.3cm}
    \subfigure{
    \begin{minipage}{0.28\textwidth}
    \centering
    % \hspace{100pt}
    \includegraphics[width=\linewidth]{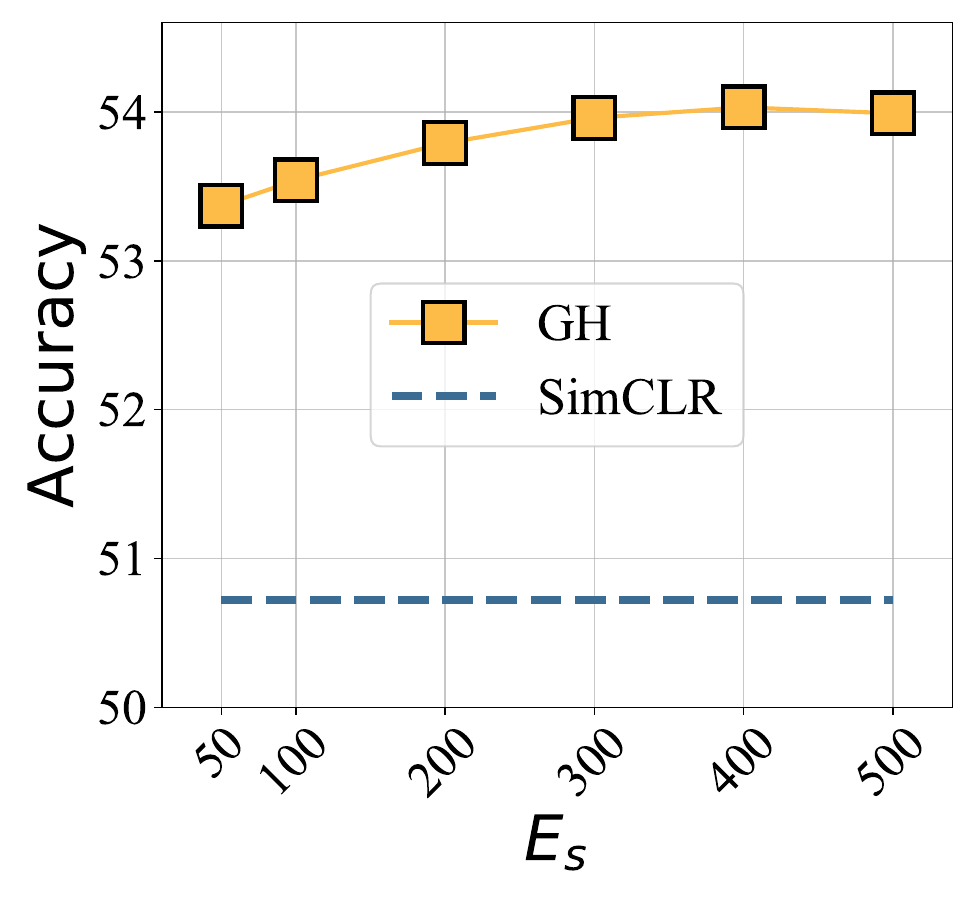}
    \end{minipage}
    }
    \vspace{-10pt}
    \caption{Ablations of temperature $\gamma_{\mathrm{GH}}$, coefficient $\lambda$ and Sinkhorn iteration $E_s$ on CIFAR-LT.}
    
    \label{fig:appendix_ablation}
\end{figure}

\subsection{Ablations on the Training Epoch}

In \cref{tab:epoch}, we present the comparison between SimCLR and SimCLR+\methodspace on CIFAR-LT-R100 under different training epochs.

\begin{table}[!htb]
\centering
\caption{Linear probing results on CIFAR-LT-R100 with different training epochs.} \label{tab:epoch}
\begin{tabular}{c|ccccc}
\toprule[1.5pt]
epoch & 200 & 500 & 1000 & 1500 & 2000 \\
\midrule
SimCLR & 49.53 & 50.32 & 50.72 & 50.84 & 50.27 \\
% \midrule
+\method & 50.89 & 54.00 & 53.96 & 53.95 & 53.91 \\
\bottomrule[1.5pt]
\end{tabular}
\end{table}

According to the table, we can see that both methods appropriately reach the saturated performance when the training epochs are larger than 500. To guarantee the converged performance, we thus set the default training epochs as 1000. 

\subsection{Ablations on the Batch Size}

To explore the effect of the training batch size, we conduct the experiments with different batch size on CIFAR-LT-R100 as follows.

\begin{table}[!htb]
\centering

\caption{Linear probing results on CIFAR-LT-R100 w.r.t the methods with different batch size.} \label{tab:batchsize}
\begin{tabular}{c|ccccc}
\toprule[1.5pt]
Batch size & 128 & 256 & 512 & 768 & 1024 \\
\midrule
SimCLR & 50.14 & 51.08 & 50.72 & 50.25 & 50.07 \\

+\method & 52.72 & 53.43 & 53.96 & 53.95 & 53.18 \\
\bottomrule[1.5pt]
\end{tabular}
\end{table}

\begin{table}[!htb]
\centering
\revise{
\caption{Linear probing results on ImageNet-LT w.r.t the methods with different batch size.} \label{tab:batchsize_imagenet}
\begin{tabular}{c|ccccc}
\toprule[1.5pt]
Batch size & 256 & 384 & 512 & 768  \\
\midrule
SimCLR & 36.65 & 36.97 & 37.85 & 38.04  \\

+\method & 38.28 & 39.22 & 41.06 & 41.34  \\
\bottomrule[1.5pt]
\end{tabular}}
\end{table}

\revise{From the results, we can see that our \methodspace consistently outperforms the baseline SimCLR. It is worth noting that our method still provides siginificant improvements when the batch size is small (e.g. 2.6\% with batch size as 128 on CIFAR-LT), which reflects the robustness of the proposed \methodspace in terms of small batch sizes. Besides, we observe that the performance drops when reducing the batch size for both baseline method and our \methodspace on CIFAR-LT and ImageNet-LT, as shown in \cref{tab:batchsize}. This can potentially be attributed to the higher probability of encountering situations where certain classes are missing under smaller batch size. Intuitively, it might easily generate biased estimation when there is no support for a certain class in the mini-batch. Then, the cluster quality might be affected by the probability of encountering missing class, which potentially correlates the important factor, \textit{i.e.}, batch size.}

\subsection{Ablations on Geometric Uniform Structure}

In \cref{tab:gus}, we conduct experiments with the geometric uniform structure as the projector on top of the baseline contrastive learning methods. As can be seen, if geometric uniform structure alone is used to balance the representation
learning, the improvement is minor and sometimes degrades. This is because the direct estimation
from the geometric uniform structure is noisy during training when the representation is not ideally distributed. 

\begin{table}[!htb]
\centering
\caption{Ablations of the geometric uniform structure on CIFAR-100-LT with different imbalanced ratios~(100, 50, 10).}
\label{tab:gus}
\renewcommand\arraystretch{1}
\begin{tabular}{lccc}
\toprule[1.5pt]
Method        & CIFAR-LT-R100                      & CIFAR-LT-R50                       & CIFAR-LT-R10                       \\ \midrule
SimCLR        & 50.72                     & 52.24                     & 55.67                     \\
+GUS  & 51.10 & 51.99 & 55.56 \\
+\method      & 53.96                     & 55.42                     & 57.36                     \\ \bottomrule[1.5pt]
\end{tabular}
\end{table}

\subsection{Ablations on the Momentum Hyper-parameter}

In our proposed \method, the hyper-parameter $\beta$ controls the smoothing degree on the historical statistics regarding the dynamically estimated surrogate label distribution $\boldsymbol{\pi}$. We conduct empirical comparison with different $\beta$ to validate the stability of our method, as depicted in \cref{fig:beta}. From the results, we can see that our \methodspace can achieve consistent performance at the most cases. To guarantee the performance, we thus set the default hyper-parameter $\beta$ as 0.999.

\begin{figure}[!htb]
	\centering
	\includegraphics[width=0.45\textwidth]{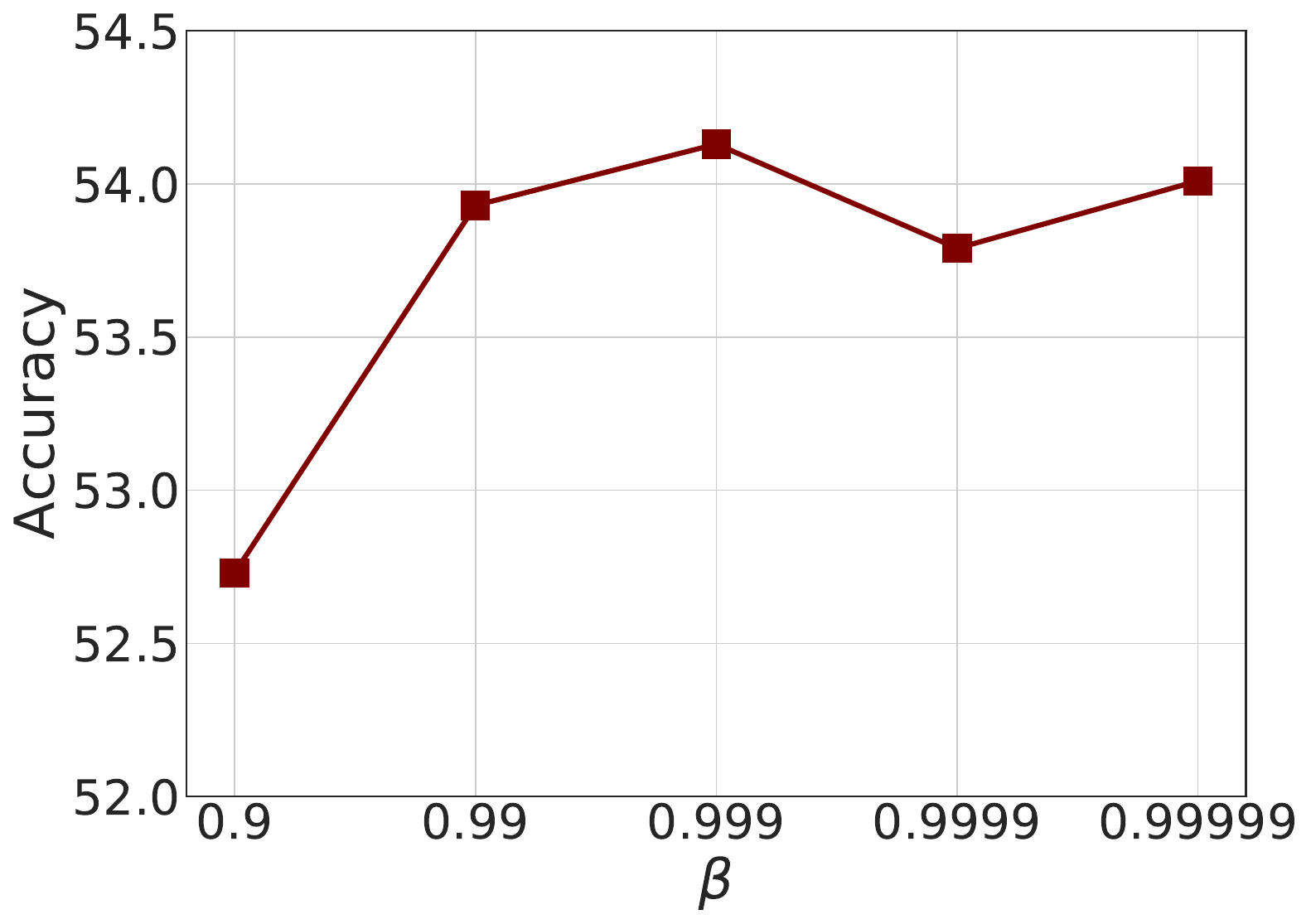}
    \caption{Linear probing \textit{w.r.t} hyper-parameter $\beta$ on CIFAR-LT-R100.}\label{fig:beta}

\end{figure}

\subsection{Implementations of Geometric Uniform Structure}\label{sec:impgus}

In \cref{fig:etfapp}, we empirically compare 
the results of analytical geometric uniform structure (Simplex ETF) with those of proxy variants. We thus conduct the sensitivity analysis w.r.t a smaller span of K, ranging from 30 to 220. From the results, we observe that the comparable performance is achieved in both geometric structures. This indicates that our method is effective to two forms of the geometric structure, relaxing the hard dimensional constraints in the analytical solution. It is also worth noting that our method's efficacy remains unaffected by the dimension of the geometric uniform structure when appropriately choosing the dimension, highlighting its ease of application.

\begin{figure}[!htb]
\centering
\includegraphics[width=0.5\textwidth]{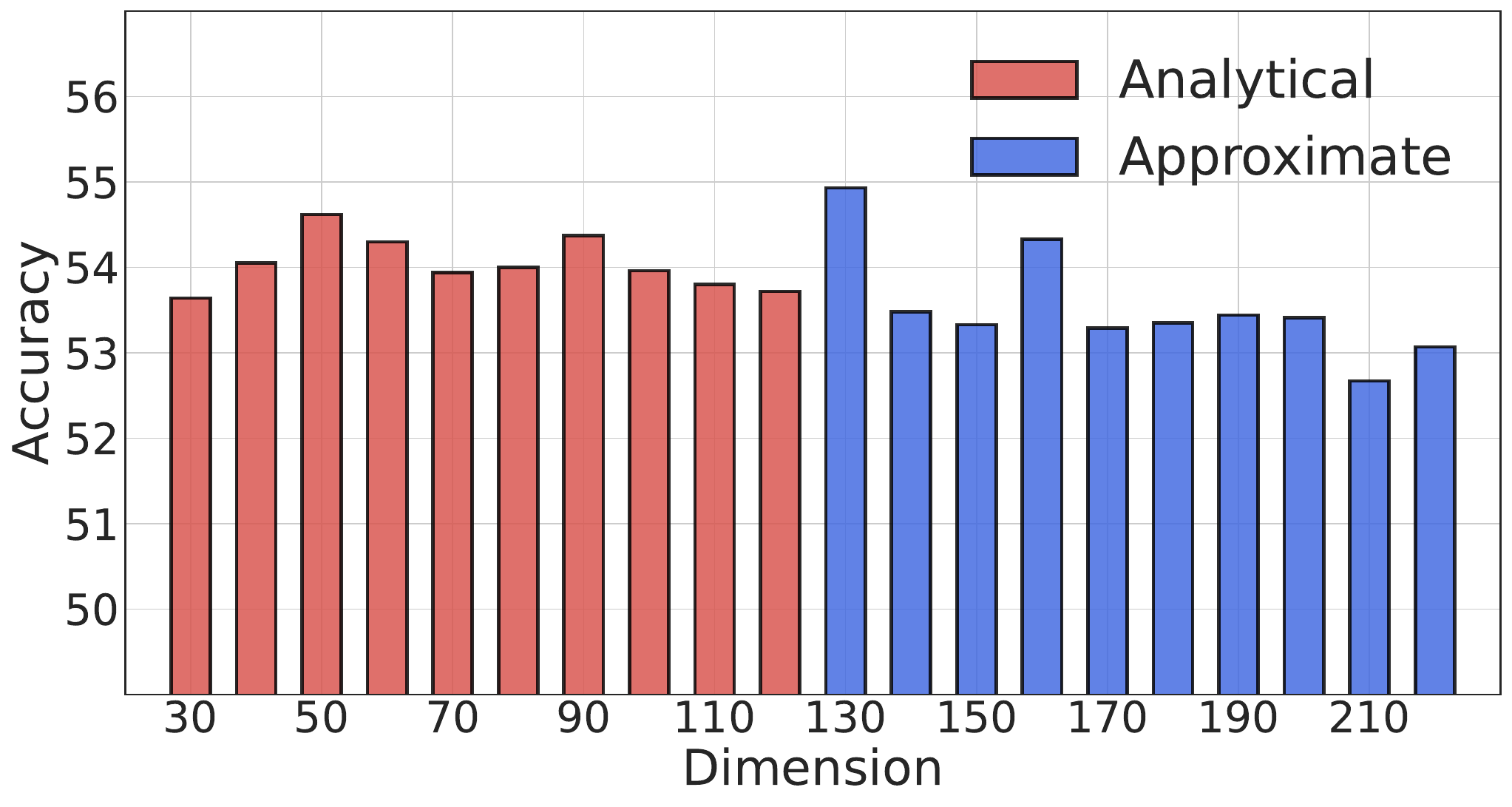}
\caption{Linear probing performance \textit{w.r.t.} the dimension $K$ of the geometric uniform structure $\mathbf{M}$ on CIFAR-LT-R100. Analytical or approximate solution are applied according to the dimensional constraints. More details can be referred to \cref{appendix:structure}.}\label{fig:etfapp}
\end{figure}

\subsection{Full-shot Evaluation on Large-scale Dataset}\label{exp:fullshot}

 Here we provide 100-shot evaluation and full-shot evaluation on ImageNet-LT, as shown in \cref{tab:full-shot}. We observe that the performance improvements and representation balancedness~(Std) are consistent with both evaluations, indicating the rationality of the 100-shot evaluation.

\begin{table}[!htb]
\centering

\caption{Full-shot linear evaluation and 100-shot evaluation on ImageNet-LT.} \label{tab:full-shot}
\begin{tabular}{llccccc}
\toprule[1.5pt]
Evaluation                 & Method   & Many  & Medium & Few   & Std  & Avg   \\ \midrule
\multirow{2}{*}{100-shot}  & SimCLR   & 41.69 & 33.96  & 31.82 & 5.19 & 36.65 \\
                           & +\method & 41.53 & 36.35  & 35.84 & 3.15 & 38.28 \\\midrule
\multirow{2}{*}{Full-shot} & SimCLR   & 42.86 & 35.17  & 33.13 & 5.13 & 37.86 \\
                           & +\method & 44.11 & 38.59  & 37.87 & 3.41 & 40.62 \\ \bottomrule[1.5pt] 
\end{tabular}
\end{table}

\subsection{Comprehensive Evaluation on More Real-world Scenarios}

To further validate the generalization of the proposed method, we conduct more comprehensive comparisons on various datasets with distinct characteristics and tasks, and conduct more experiments as follows: 

\begin{itemize}
    \item Marine-tree dataset~\citep{boone2022marine}:  This dataset is a large-scale dataset for marine organism classification. It contains more than 160K images divided into 60 classes with the number of images per class ranging from 14 to 16761.
    \item IMDB-WIKI-DIR dataset~\citep{yang2021delving}: IMDB-WIKI-DIR (age) dataset is subsampled from IMDB-WIKI dataset~\citep{rothe2018deep} to construct the deep imbalanced regression benchmark. It contains 202.5K images with the number of images per bin varied between 1 and 7149. 
    \item CUB-200~\citep{wah2011caltech} and Aircrafts~\citep{maji2013fine} dataset. Caltech-UCSD Birds 200 (CUB-200) and Aircrafts dataset are two fine-grained datasets, which contains 11K images with 200 classes and 10K images with 102 classes, respectively.
\end{itemize}

\begin{table}[!htb]
\centering

\caption{Linear probing results (average accuracy, \%) on Marine-tree dataset.} \label{tab:marine}
\begin{tabular}{c|cccc}
\toprule[1.5pt]
Marine & Many & Medium & Few & Avg \\
\midrule
SimCLR & 36.05 & 47.01 & 48.80 & 43.95 \\
+\method & 35.70 & 47.14 & 51.62 & 44.82 \\
\bottomrule[1.5pt]
\end{tabular}
\end{table}

\begin{table}[!htb]
\centering

\caption{Vanilla finetuning results under the metric of mean average error (MAE~\citep{yang2021delving}, \textbf{lower is better}) on IMDB-WIKI-DIR dataset.} \label{tab:mae}
\begin{tabular}{c|cccc}
\toprule[1.5pt]
IMDB-WIKI-DIR (MAE) & Many & Medium & Few & Avg \\
\midrule
SimCLR & 8.10 & 18.31 & 29.99 & 9.14 \\
% SimCLR & 8.38 & 19.61 & 31.72 & 9.52 \\
+\method & 7.77 & 17.18 & 29.29 & 8.75 \\
\bottomrule[1.5pt]
\end{tabular}
\end{table}

\begin{table}[!htb]
\centering

\caption{Vanilla finetuning results under the metric of Geometric Mean (GM~\citep{yang2021delving}, \textbf{lower is better}) on IMDB-WIKI-DIR dataset.} \label{tab:gm}
\begin{tabular}{c|cccc}
\toprule[1.5pt]
IMDB-WIKI-DIR (GM) & Many & Medium & Few & Avg \\
\midrule
SimCLR & 4.87 & 15.01 & 26.61 & 5.43 \\
% SimCLR & 5.17 & 17.01 & 28.89 & 5.79 \\
+\method & 4.64 & 13.49 & 24.54 & 5.14 \\
\bottomrule[1.5pt]
\end{tabular}
\end{table}

\begin{table}[!htb]
\centering

\caption{Downstream linear probing results (Top1/Top5 accuracy, \%) on CUB-200 dataset.} \label{tab:cub}
\resizebox{\textwidth}{!}{
\begin{tabular}{c|cc|cc|cc|cc|cc}
\toprule[1.5pt]
CUB-200 & SimCLR & +\method & Focal & +\method & SDCLR & +\method & DnC & +\method & BCL & +\method \\
\midrule
TOP1 & 28.97 & 29.89 & 30.13 & 30.68 & 28.98 & 29.63 & 29.64 & 30.46 & 28.46 & 28.97 \\
TOP5 & 57.28 & 57.92 & 58.01 & 58.78 & 57.34 & 57.95 & 57.63 & 58.55 & 56.92 & 57.66 \\
\bottomrule[1.5pt]
\end{tabular}}
\end{table}

\begin{table}[!htb]
\centering

\caption{Downstream linear probing results (Top1/Top5 accuracy, \%) on Aircrafts dataset.} \label{tab:aircraft}
\resizebox{\textwidth}{!}{%
\begin{tabular}{c|cc|cc|cc|cc|cc}
\toprule[1.5pt]
Aircrafts & SimCLR & +\method & Focal & +\method & SDCLR & +\method & DnC & +\method & BCL & +\method \\
\midrule
TOP1 & 29.82 & 30.63 & 31.02 & 31.74 & 30.99 & 31.85 & 31.18  & 32.05 & 32.79 & 35.88 \\
TOP5 & 56.14 & 57.95 & 57.82 & 58.99 & 58.09 & 59.13 & 58.11 & 59.42 & 60.79 & 63.34 \\
\bottomrule[1.5pt]
\end{tabular}}
\end{table}

On large-scale dataset (Marine-tree dataset and IMDB-WIKI-DIR dataset), we adopt the training schedule similar to ImageNet-LT and Places-LT, except the training epochs reduced from 500 to about 200 epochs. Besides, we crop the images with the low resolution (112x112) to speed up the training. We conduct linear probing on Marine-tree, CUB-200 and Aircrafts dataset. The former is pretrained with Marine-tree dataset, while the latter is pretrained with ImageNet-LT. As for IMDB-WIKI-DIR dataset, we pretrain the network for initializing the weights of the downstream supervised imbalanced regression task. Specially, the geometric mean (GM) is defined for better prediction fairness~\citep{yang2021delving}. Both the evaluation metrics (MAE, GM) are the smaller the better.

From the results in \cref{tab:marine,tab:mae,tab:gm,tab:cub,tab:aircraft}, we can see that our proposed \methodspace consistently outperforms the baseline methods for all the metrics (linear probing accuracy, finetuning accuracy, MAE and GM) on various datasets/settings. This indicates the potential of \methodspace for adapting to a wide range of real-world data scenarios to counteract the negative impact of the long-tailed distribution.

\subsection{More Results on Joint Optimization with Warm-up Strategy}

We can potentially adopt warm-up strategy to initialize the weights $\theta$ against the degenerate solutions in the joint optimization. In this subsection, we conduct more comprehenvive experiments on CIFAR-LT-100 with different warm-up epochs to further verify the superiority of the proposed bi-level optimization.

\begin{table}[!htb]
\centering

\caption{Linear probing results of joint optimization on CIFAR-LT-R100 with different warm-up epochs.} \label{tab:joint_apx}
\begin{tabular}{c|ccccccc}
\toprule[1.5pt]
Epoch & 0 & 10 & 50 & 100 & 200 & 300 & 400 \\
\midrule
Accuracy & 50.18 & 51.14 & 50.77 & 50.97 & 50.44 & 50.21 & 50.57\\
\bottomrule[1.5pt]
\end{tabular}
\end{table}

From the results in \cref{tab:joint_apx}, we can see that the warm-up strategy has the potential to improve the linear probing performance by 1\% over the vanilla joint training. However, it seems that this strategy is sensitive to the proper epochs for warming-up, and the overall performance is not better than the bi-level optimization.

\section{Broader Impacts}

Learning long-tailed data without annotations is a vital element in the deployment of robust deep learning systems in the real-world applications~\citep{hong2023long,zhou2023balanced,zhang2021exploiting,zhang2021complementary}. The attribution is that real-world natural resources inevitably exhibit the long-tailed distribution~\citep{reed2001pareto}. The importance of self-supervised long-tailed learning is further emphasized when extended to a range of safety-critical scenarios~\citep{chen2023enhanced,huang2022registration,huang2022self}, including medical intelligence~\citep{wu2022integrating,wu2023towards,zhang2023grace}, autonomous driving~\citep{sun2021three, sun2022human, sun2023modality} and criminal surveillance, where the data imbalance may lead to the distorted representation. In this paper, we study a general and practical research problem in representation learning parity for self-supervised long-tailed learning, considering the intrinsic limitation of conventional contrastive learning that can not adequately address the over-expansion of the majorities and the passive-collapse of the minorities in the embedding space. Our method regularizes long-tailed learning from a geometric perspective and motivates more benign representation, which helps improve the downstream generalization and representation balancedness. \revise{Besides, our method has the potential to be applied in fairness research scenarios~\citep{liu2017sphereface} where both majority and minority classes~(or attributes) are present. Given the guidance of label information, we can explicitly constrain a consistent embedding space for each subgroup, thereby promoting category-level uniformity.}

Nevertheless, it is important to acknowledge that our method may have negative impact, such as employment disruption, as our study endeavors to reduce annotation costs by enabling robust self-supervised learning on hard-to-collect tail data resources. Specially, if self-supervised learning can extract the tail distribution with sufficient accuracy, the necessity for the human manipulation on the quality of data distribution will diminish. 

% However, our method may imply negative impacts 

\section{Limitations}\label{appendix:limitations}

Roughly, our design is built upon the intrinsic clustering patterns that can inclusively represent the information for the downstream tasks. Although we demonstrate the appealing performance in the current benchmark, it cannot be always guaranteed in all scenarios. Once such a condition is not satisfied, namely, clustering only captures the task-irrelevant patterns but ignores the task-relevant details, the improvement might be limited or even negative. A potential way to overcome this drawback is using a small auxiliary labeling set to calibrate the clustering dynamic aligned with the downstream tasks, namely, a semi-supervised paradigm. The methods to encourage learning the stable features in the area of causal inference can also be borrowed to this problem to alleviate this dilemma.

%%%%%%%%%%%%%%%%%%%%%%%%%%%%%%%%%%%%%%%%%%%%%%%%%%%%%%%%%%%%%%%%%%%%%%%%%%%%%%%
%%%%%%%%%%%%%%%%%%%%%%%%%%%%%%%%%%%%%%%%%%%%%%%%%%%%%%%%%%%%%%%%%%%%%%%%%%%%%%%

\end{document}